\documentclass[runningheads]{llncs}

\usepackage{eccv}

\usepackage{eccvabbrv}

\usepackage{graphicx}
\usepackage{booktabs}

\usepackage[accsupp]{axessibility}  

\usepackage{hyperref}

\usepackage{orcidlink}

\usepackage{algorithm}
\usepackage{algorithmicx}
\usepackage{algpseudocode}
\usepackage{multirow}
\usepackage{makecell}
\usepackage{capt-of}
\usepackage{booktabs}   
\usepackage{graphicx}   
\usepackage{wrapfig}
\usepackage{float}
\usepackage{xcolor}
\usepackage{multirow}

\usepackage{amsmath,amssymb,mathtools}
\newcommand{\vn}{\mathbf{n}}
\newcommand{\vd}{\mathbf{d}}
\newcommand{\vv}{\mathbf{v}}

\newcommand{\mN}{\mathbf{N}}
\newcommand{\mD}{\mathbf{D}}
\newcommand{\mX}{\mathbf{X}}

\begin{document}

\title{TetraSDF: Analytic Isosurface Extraction with Multi-resolution Tetrahedral Grid
}
\titlerunning{TetraSDF}


\author{
Seonghun Oh\inst{1}\textsuperscript{*}\orcidlink{0009-0003-1578-9446}
\and
Youngjung Uh\inst{1}\orcidlink{0000-0001-8173-3334}
\and
Jin-Hwa Kim\inst{2,3}\textsuperscript{\ensuremath{\dagger}}\orcidlink{0000-0002-0423-0415}
}

\authorrunning{S. Oh et al.}

\institute{
Yonsei University, Republic of Korea
\and
NAVER AI Lab, Republic of Korea
\and
Seoul National University, Republic of Korea
\\
\email{\{rendell,yj.uh\}@yonsei.ac.kr, j1nhwa.kim@navercorp.com}
}

\maketitle

\begingroup
\renewcommand{\thefootnote}{}
\footnotetext{
\textsuperscript{*}Work done during an internship at NAVER AI Lab.
\quad
\textsuperscript{\ensuremath{\dagger}}Corresponding author.
\quad
Project page: \url{https://seonghunn.github.io/tetrasdf/}
}
\endgroup

\begin{abstract}
Extracting an explicit surface that exactly matches the zero-level set of a neural signed distance function (SDF) remains challenging.
Sampling-based isosurfacing methods such as Marching Cubes introduce discretization error.
In contrast, continuous piecewise affine (CPWA) analytic approaches typically require plain ReLU MLPs, which limits the ability to learn high-frequency SDFs in practice.
We present TetraSDF, an analytic isosurface extraction framework for SDFs that retains the expressiveness of grid-based encoders while enabling exact zero-level set extraction, by representing the SDF with a ReLU MLP composed with a multi-resolution tetrahedral positional encoder.
Our positional encoder's barycentric interpolation preserves a global CPWA structure, allowing us to track ReLU linear regions within an encoder-induced polyhedral complex.
We further introduce a fixed analytic input preconditioner derived from the encoder’s metric to reduce directional bias, thereby stabilizing training.
Across multiple benchmarks, TetraSDF matches or surpasses existing grid-based encoders in SDF reconstruction accuracy, while faithfully recovering the network’s zero-level set as a triangle mesh.
\end{abstract}

\section{Introduction}
\label{sec:intro}
Extracting explicit isosurfaces (\eg, triangle meshes) from learned SDFs is a standard step in neural implicit surface pipelines.
Learned SDFs are widely used for distance- and collision-based queries, while many downstream tools and pipelines consume explicit surface representations, making isosurface extraction a common interface step.
Ensuring that the extracted surface matches the SDF’s zero-level set helps preserve geometric consistency.
As a result, downstream objectives defined on the surface, such as normal alignment to field-derived normals or collision checking against the extracted surface, remain faithful to the network’s isosurface representation.

Classical methods such as Marching Cubes~\cite{Lorensen1987MarchingCubes}, Marching Tetrahedra~\cite{Treece1999RMT}, and Dual Contouring~\cite{Ju2002DualContouring} discretize the implicit field defining the isosurface, inevitably losing geometric fidelity due to their resolution-dependent sampling.
The resulting mismatch between the extracted mesh and the network’s isosurface cannot be eliminated at finite sampling resolution, and reducing the mismatch inflates triangle budgets and may introduce staircase artifacts.
Analytic Marching~\cite{Lei2020AnalyticMarching} later addressed this issue by interpreting ReLU-based MLPs as continuous piecewise affine (CPWA) functions~\cite{Montufar2014LinearRegions,Raghu2017ExpressivePower,Serra2018Bounds,Hanin2019Complexity}, allowing exact isosurface extraction from the network’s zero-level set.
However, its complexity scales poorly with network expressiveness, making it impractical.
Subsequent work refines this with an edge subdivision~\cite{berzins2023polyhedral} that leverages the sign vectors to identify the parent cells of boundary vertices without iterating all cells, improving both precision and efficiency.

Although these analytic methods yield explicit isosurfaces that exactly match the networks’ zero-level sets, they are limited to ReLU MLPs where CPWA analysis applies. Such networks are prone to spectral bias~\cite{rahaman2019spectral,Tancik2020FourierFeatures}, making it difficult to learn complex SDFs.
TropicalNeRF~\cite{Kim2024} attempts to retain the expressiveness of grid-based positional encoders while enabling analytic extraction via a piecewise trilinear interpretation.
However, for grid-based encoders using trilinear interpolation~\cite{mueller2022instant,chen2022tensorf}, its analytic guarantee relies on enforcing an eikonal constraint within each trilinear cell, which is difficult to satisfy uniformly, and thus uses a diagonal-plane intersection heuristic; consequently, the extracted isosurface is not guaranteed to exactly match the network’s zero-level set in general.


To overcome the precision loss caused by non-affine interpolation while still avoiding spectral bias, we introduce a network architecture that combines a multi-resolution tetrahedral positional encoder with a ReLU MLP, exploiting the affine nature of barycentric interpolation. 
This design preserves the benefits of grid-based positional encoders for learning high-frequency SDFs, while ensuring that the overall mapping remains CPWA and thus amenable to analytic isosurface extraction.
Notice that, unlike axis-aligned cuboids, the tetrahedral encoder partitions space into \emph{polyhedral cells}, which are then further subdivided by the folded hyperplanes of the ReLU MLP.

Building on this observation, we propose a novel analytic isosurface extraction framework for networks employing our multi-resolution tetrahedral positional encoder with a ReLU MLP.
While the positional encoder enhances the shape expressiveness, the extracted isosurface still exactly matches the network’s zero-level set due to the CPWA nature of the overall mapping.
The encoder is designed to support efficient cell indexing and neighbor identification, which are essential for analytic traversal of the polyhedral structure. 

Furthermore, we extend the edge subdivision algorithm to operate jointly on encoder-induced polyhedral cells and ReLU MLP linear regions, enabling exact CPWA analysis under positional encoding.
We also introduce an analytic input preconditioner for the multi-resolution tetrahedral encoder that mitigates directional bias induced by the encoder’s grid geometry, improving optimization conditioning during training.
In summary, our contributions are as follows:
    


\begin{itemize}
\item We propose TetraSDF, a framework that bridges expressive SDF learning and exact analytic isosurface extraction by preserving a global CPWA structure via barycentric interpolation in a tetrahedral positional encoder.
\item We propose a tensorized scheme for constructing and handling the resulting polyhedral structures, enabling an efficient GPU implementation.
\item We characterize directional bias induced by the encoder’s geometry via its metric and derive an input preconditioner to mitigate it, improving SDF learning accuracy.
\end{itemize}
\section{Related Work}
\label{sec:related}
\subsection{Isosurface Extraction from Neural Implicit Fields}
In neural implicit representations, shapes are modeled as continuous scalar fields parameterized by neural networks~\cite{Park2019DeepSDF,Mescheder2019Occupancy}.
Classical isosurface extraction methods such as Marching Cubes and its variants~\cite{Lorensen1987MarchingCubes, Treece1999RMT, Ju2002DualContouring} convert these fields into meshes via sampling.
However, the resulting surfaces suffer from discretization error, and reducing this error requires dense sampling, which rapidly increases mesh complexity.
Recent finite-sample surface reconstruction and adaptive isosurfacing methods further improve explicit surface recovery from sampled or queried fields~\cite{sellan2023reach,kohlbrenner2025power,ren2024mcgrids}.
While effective, these methods still recover surfaces from finite or adaptive field evaluations, rather than directly from the continuous neural field represented by the trained network itself.
To improve fidelity or enable learning-based extraction, several works integrate neural networks with isosurface extraction or differentiable mesh parameterizations~\cite{chen2021nmc,chen2022ndc,shen2021dmtet,gao2020deftet,shen2023flexicubes}.
These approaches can produce higher-quality meshes and support end-to-end training, but still rely on discrete sampling or surrogate volumetric parameterizations as the basis for mesh extraction.
Analytic meshing methods instead exploit the CPWA structure of ReLU MLPs~\cite{Lei2020AnalyticMarching,berzins2023polyhedral,stippel2025marching} or piecewise trilinear formulations~\cite{Kim2024} to more directly characterize the network’s zero-level sets, avoiding reliance on sampling.
However, in practice, the former are restricted to plain MLPs with limited ability to represent complex SDFs, while the latter relies on geometric heuristics that can degrade precision.

\subsection{Positional Encoding Methods}
Learning high-frequency geometric details in neural implicit functions is often hindered by spectral bias~\cite{rahaman2019spectral,Tancik2020FourierFeatures}. To mitigate this effect, positional encoding schemes have been proposed to map input coordinates into higher-dimensional feature spaces with rich frequency content~\cite{Tancik2020FourierFeatures,sitzmann2020implicit} to overcome spectral bias.
Beyond frequency-based methods, grid-based positional encoders have become dominant for large-scale 3D scenes, where learnable features are stored in spatial grids and interpolated at query points \cite{liu2020neural,chen2022tensorf,mueller2022instant,fridovich2023k}. 
Among grid-based approaches, Tetra-NeRF~\cite{kulhanek2023tetra} employs an adaptive tetrahedral representation and PermutoSDF~\cite{rosu2023permutosdf} leverages a permutohedral lattice; both use barycentric-style interpolation of features within simplex cells.
However, the data-dependent tetrahedralization in Tetra-NeRF and the higher-dimensional embedding in PermutoSDF make explicit cell and neighbor indexing less straightforward for analytic extraction.

\subsection{Preconditioning}
Preconditioning is a classical technique from numerical linear algebra~\cite{Greenbaum1997,Saad2003} that improves the convergence of iterative solvers by transforming ill-conditioned systems into better-conditioned forms.
This concept has been widely adopted in deep learning to stabilize training and accelerate convergence through gradient or parameter preconditioning via adaptive per-parameter scaling~\cite{Kingma2015Adam,Tieleman2012RMSProp,Gupta2018}.
ZCA whitening~\cite{Zeiler2014} similarly normalizes the covariance of input features to promote more isotropic learning behavior in the data space.
Recent works further explore preconditioning in structured domains such as implicit neural fields and camera parameter spaces~\cite{chng2025preconditioners,park2023camp}.
In the same spirit, we derive a fixed analytic input preconditioner that whitens the encoder-induced metric, thereby removing the directional bias introduced by the encoder’s geometric structure.

\begin{figure*}[ht]
  \centering
  \includegraphics[width=0.6\textwidth]{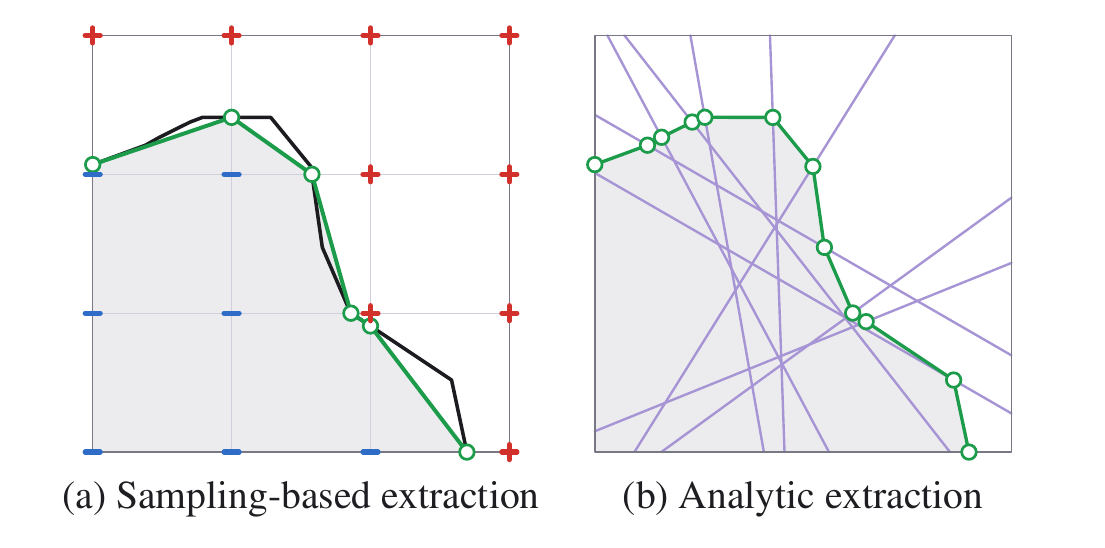}
    \caption{A schematic 2D illustration of sampling-based and analytic isosurface extraction for CPWA functions, such as ReLU MLPs. Sampling-based extraction approximates the network's zero-level set from finitely sampled field values, whereas analytic extraction recovers the zero-level set by exploiting the CPWA structure. Black denotes the network's zero-level set, and green denotes the extracted surface.}
  \label{fig:sampling-vs-analytic}
\end{figure*}

\section{Preliminary}
\subsection{Definitions}
\label{sec:tetra-networks}
\begin{definition}[Tetrahedral networks]
A \emph{tetrahedral network} $f = \nu^{(M)}\! \,\circ\, \tau$ is the composition of a multi-resolution tetrahedral positional encoder 
$\tau$ and a $M$-layer ReLU network $\nu^{(M)}$. $\nu^{(M)}$ is defined as follows:
\begin{equation}
\nu^{(M)} = \rho^{(M)} \, \circ \,\sigma^{(M-1)} \,\circ \,\rho^{(M-1)} \,\circ \,\cdots \,\circ \,\sigma^{(1)} \,\circ \,\rho^{(1)}.
\label{eq:tetrahedral-network}
\end{equation}
Here, $\sigma(\mathbf{x})=\max(\mathbf{x},0)$ and $\rho^{(i)}(\mathbf{x})=W^{(i)}\mathbf{x}+b^{(i)}$.
\end{definition}

\begin{remark}[Piecewise affine property of tetrahedral networks] 
The encoder $\tau$ maps an input $\mathbf{x}\in\mathbb{R}^3$ to a feature vector by performing barycentric interpolation at each level of the multi-resolution tetrahedral positional encoder and concatenating the resulting feature vectors. Since barycentric interpolation is affine, the interpolated feature vector at each level is affine in $\mathbf{x}$ within the containing tetrahedron. Moreover, because the concatenation of affine functions is affine, the output of $\tau$ remains affine on such regions (for detailed proofs, see Appendix~\ref{appx:theory}). Consequently, $\tau$ defines a piecewise affine mapping over the input domain. The ReLU network $\nu$ is also piecewise affine in its input, so its composition, the tetrahedral network $f$ is piecewise affine \wrt $\mathbf{x}$.
\end{remark}


\begin{definition}[Sign vectors]
\label{sec:sign-vectors}
Let \(\nu^{(m)}_k\) denote the \(k\)-th neuron in the \(m\)-th layer of an \(M\)-layer ReLU network \(\nu^{(M)}\), and let \(\nu^{(m)}_k(\mathbf{x})\) denote its pre-activation value at input \(\mathbf{x}\).
The \emph{sign vector} associated with $\mathbf{x}$ at that neuron is defined as:
\begin{equation}
\mathbf{s}^{(m)}_k(\mathbf{x}) = [\,\gamma^{(1)}_1(\mathbf{x}),\, \gamma^{(1)}_2(\mathbf{x}),\, \dots,\, \gamma^{(m)}_k(\mathbf{x})\,],
\end{equation}
where each scalar entry $\gamma^{(i)}_j(\mathbf{x})$ is defined using a small positive constant $\epsilon_s$ as:
\begin{equation}
\gamma^{(i)}_j(\mathbf{x}) =
\begin{cases}
+1, & \text{if } ~\nu^{(i)}_{j}(\mathbf{x})~ > ~~\epsilon_s,\\
~~~0, & \text{if } |\nu^{(i)}_{j}(\mathbf{x})| \le ~~\epsilon_s,\\
-1, & \text{if } ~\nu^{(i)}_{j}(\mathbf{x})~ < -\epsilon_s.
\end{cases}
\end{equation}
Here, $\gamma^{(i)}_j(\mathbf{x}) = 0$ indicates that $\mathbf{x}$ lies on the decision boundary $\mathcal{H}^{(i)}_j = \{\mathbf{x} \mid \nu^{(i)}_j(\mathbf{x}) = 0\}$. The values $+1$ and $-1$ correspond to the two half-spaces, \ie, the two linear regions on either side of this boundary.
\end{definition}

\subsection{Edge Subdivision}
\label{sec:edge-subdivision}
Through successive ReLUs, each decision boundary $\mathcal{H}^{(m)}_k$ becomes a \emph{folded hyperplane}~\cite{grigsby2022transversality,berzins2023polyhedral} in the input space.
Starting from an initial skeleton with vertices $\mathcal{V}$ and edges $\mathcal{E}$ (\eg, a cube bounding the scene), we sequentially process all boundaries $\mathcal{H}^{(m)}_k$ in increasing order of $(m,k)$, updating $\mathcal{V}$ and $\mathcal{E}$ at each step.
For each boundary $\mathcal{H}^{(m)}_k$, every intersected edge is subdivided.
A sign change of $\gamma^{(m)}_k$ at the two endpoints indicates that the edge crosses $\mathcal{H}^{(m)}_k$ and that its vertices lie in different linear regions. 
For an edge $(\mathbf{x}_0, \mathbf{x}_1)$ satisfying $\gamma^{(m)}_k(\mathbf{x}_0)\,\gamma^{(m)}_k(\mathbf{x}_1) < 0$, let $\nu^{(m)}_k(\mathbf{x}_0) = d_0$ and $\nu^{(m)}_k(\mathbf{x}_1) = d_1$. 
Then, the intersection is computed as:
\begin{equation}
\hat{\mathbf{x}}_{0,1} = (1 - w)\mathbf{x}_0 + w\mathbf{x}_1,
\quad 
w = |d_0|/|d_0 - d_1|.
\end{equation}
The intersection point $\hat{\mathbf{x}}_{0,1}$, satisfying $\gamma^{(m)}_k(\hat{\mathbf{x}}_{0,1}) = 0$, 
is added to $\mathcal{V}$, and the corresponding edge in $\mathcal{E}$ is split into two edges, $(\mathbf{x}_0, \hat{\mathbf{x}}_{0,1})$ and $(\hat{\mathbf{x}}_{0,1}, \mathbf{x}_1)$.

The tricky part, however, is that, when a folded hyperplane subdivides linear regions (convex polyhedra), the induced intersection polygon determines which new edges must be added to $\mathcal{E}$ (see \cref{fig:main_method}c).
To handle this consistently without explicitly enumerating all linear regions, we employ the \emph{sign vectors}~\cite{berzins2023polyhedral} defined in \cref{sec:sign-vectors}.
A candidate edge formed by two previously computed intersection points is considered valid only when the following conditions are satisfied:
First, their sign vectors must match on all nonzero entries, ensuring that both points belong to the same linear region. Notice that a zero entry is treated as a wildcard that can agree with either $-1$ or $+1$ (for the details, see \cref{sec:preliminary-perturbation}).
Second, if a vertex pair shares at least two zero entries in their sign vectors, the pair is connected, since both vertices lie on the same pair of hyperplanes. For a high-level intuition and a step-by-step summary, see Appendix~\ref{supp:edge-subdiv}.

Notice that, when integrated with our multi-resolution tetrahedral positional encoder (\cref{sec:mtd_tetraencoder}), we augment the sign vector by appending the grid-based region indicator (\cref{sec:region-indicators}) to maintain consistency, while the rest of the procedure remains unchanged.

\subsection{Perturbation with Sign Vectors}
\label{sec:preliminary-perturbation}
Perturbing each zero entry in a sign vector to both $+1$ and $-1$ generates
the sign vectors of all neighboring parent linear regions.
This provides a systematic way to explore local adjacency relationships between regions.
We later combine this perturbation scheme with the grid-based region indicators introduced in \cref{sec:perturbation}, enabling efficient adjacency tracking across the tetrahedral network structure.
\section{Method}
We introduce the design of the multi-resolution tetrahedral positional encoder \(\tau\) and its input preconditioner.
The subsequent sections describe our mesh extraction pipeline, which consists of 1) constructing the initial skeleton, 2) performing grid-aware edge subdivision, and 3) reconstructing the surface.
\begin{figure*}[t]
  \centering
  \includegraphics[width=1.0\textwidth]{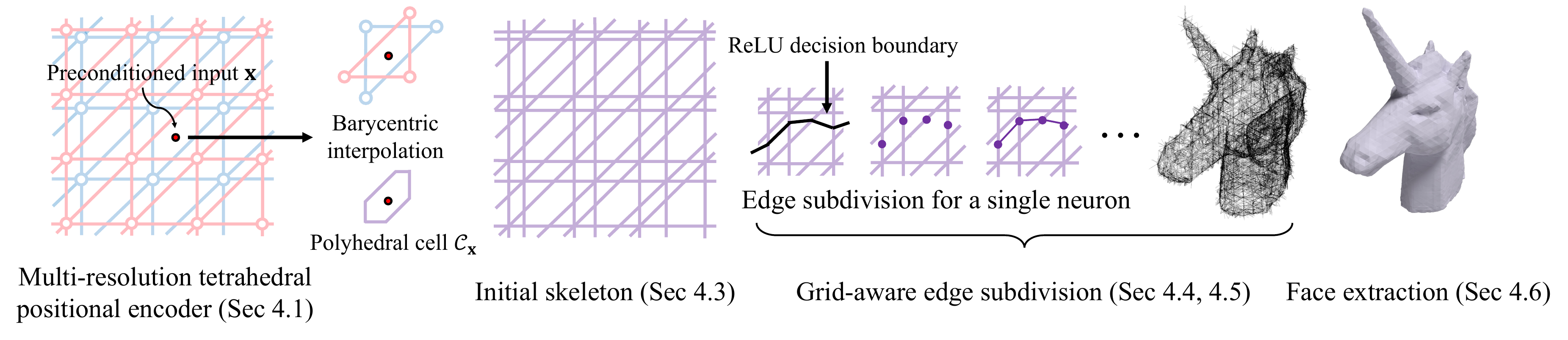}%
  \vspace{-1em}
  \caption{
  Overview of \emph{TetraSDF}. A preconditioned input $\mathbf{x}$ is mapped by the multi-resolution tetrahedral positional encoder to interpolated features within its containing polyhedral cell $\mathcal{C}_{\mathbf{x}}$ (\cref{sec:mtd_tetraencoder,sec:precondition}). 
  These cells form the encoder-induced polyhedral complex (the initial skeleton), from which we start edge subdivision (\cref{sec:grid-skeleton}).
  We then perform grid-aware edge subdivision that jointly tracks polyhedral cells and ReLU MLP linear regions to obtain the candidate vertex and edge sets $(\mathcal{V}, \mathcal{E})$ (\cref{sec:region-indicators,sec:perturbation}), and finally extract mesh faces from $(\mathcal{V}, \mathcal{E})$ to obtain the final mesh (\cref{sec:face-extraction}). In (a)–(c), we zoom into a single edge subdivision iteration for a neuron.
  }
  \label{fig:main_method}
  \vspace{-1.0em}
\end{figure*}

\subsection{Multi-resolution Tetrahedral Positional Encoder}
\label{sec:mtd_tetraencoder}
Grid-based positional encoders are a common approach to mitigate spectral bias for ReLU MLPs, but a widely used design choice based on trilinear interpolation~\cite{mueller2022instant,chen2022tensorf} breaks the CPWA structure and hinders analytic extraction.
Our multi-resolution tetrahedral encoder replaces it with barycentric interpolation, preserving high-frequency modeling while keeping the overall mapping CPWA.
\begin{definition}[Multi-resolution tetrahedral  grid]
\label{mtd:tet-enc}
Let $L \in \mathbb{N}$ be the number of resolution levels, and $N_{\min}$ and $ N_{\max}$ be the minimum and maximum grid resolutions, respectively, where $N_{\min} < N_{\max}$. 
Then, the geometric progression ratio $\gamma$ is defined as follows~\cite{mueller2022instant}:
\begin{equation}
    \gamma := \exp\!\left( \frac{\ln N_{\max} - \ln N_{\min}}{L-1} \right),
\end{equation}
while the $\ell$-level resolution $N_\ell$ is defined as follows:
\begin{equation}
    N_\ell := \left\lfloor N_{\min}\, \gamma^{\ell} \right\rfloor,
    \quad \ell \in \{0,1,\dots,L-1\}.
    \label{eq:mtd:multi-tet-res}
\end{equation}
\end{definition}
At each level $\ell$, the unit cube $\mathcal{S}\in[0,1]^3$ is uniformly subdivided into $N_\ell^3$ cube cells, each of which is further decomposed into six congruent tetrahedra following the tetrahedral convention~\cite{burstedde2016tetrahedral} in~\cref{fig:tetsubdiv_mtd}.
We further discuss the design rationale in Appendix~\ref{sec:appendix_6tet}.

\begin{definition}[Tetrahedral positional encoder]
For a query point $\mathbf{x} \in \mathcal{S}$, let $\mathcal{T}^{(\ell)}(\mathbf{x})$ be 
any\footnote{ Singular cases where $\mathbf{x}$ lies on the boundaries are negligible, as they are handled implicitly by our interpolation scheme.} tetrahedron at level $\ell$ that contains $\mathbf{x}$, 
and $V(\mathcal{T}^{(\ell)}(\mathbf{x}))$ be its set of four vertices.
Each vertex $v \in V(\mathcal{T}^{(\ell)}(\mathbf{x}))$
is mapped to the index $i = h(v)$ for the hash table via the spatial hash function 
$h$~\cite{teschner2003optimized}. 
The entry of index $i$ corresponds to the learnable feature vector
$H^{(\ell)}_{i} \in \mathbb{R}^d$. 
The barycentric weight\footnote{
Barycentric coordinates specify a point with respect to the vertices of a simplex 
by expressing the point as a convex combination of the vertices,
$\mathbf{x}=\sum_i w_i v_i$ with $\sum_i w_i = 1$ and $w_i \ge 0$. 
The tuple $(w_0,\dots,w_n)$ is called the barycentric coordinates of $\mathbf{x}$, 
and each coefficient $w_i$ is the barycentric weight associated with the vertex $v_i$.
}
 of $\mathbf{x}$ with respect to $v$ is denoted by 
$w^{(\ell)}_v(\mathbf{x})$.
Then, the tetrahedral positional encoder
$
\tau : \mathbb{R}^3 \rightarrow \mathbb{R}^{dL} 
$
is defined as:
\begin{equation}
    \tau(\mathbf{x})
    :=
    \bigoplus_{\ell=0}^{L-1}
    \left(
        \sum_{v \in V(\mathcal{T}^{(\ell)}(\mathbf{x}))}
        w^{(\ell)}_v(\mathbf{x}) \, H^{(\ell)}_{i}
    \right),
\end{equation}
where $\oplus$ denotes concatenation over all resolution levels.
\end{definition}

\begin{definition}[Polyhedral cells]
For $\mathbf{x} \in \mathcal{S}$ and each level $\ell$, the choice of $\mathcal{T}^{(\ell)}(\mathbf{x})$ determines which feature vectors are selected at that level.
We define the \emph{polyhedral cell} of $\mathbf{x}$ as the maximal region on which this combination of selected feature vectors across all levels remains constant:
\begin{equation}
    \mathcal{C}_{\mathbf{x}}
    :=
    \bigcap_{\ell=0}^{L-1} \mathcal{T}^{(\ell)}(\mathbf{x}).
\end{equation}
where $\bigcap_{\ell=0}^{L-1}$ denotes the geometric intersection over all resolution levels $\ell$ of $\mathbf{x}$-containing tetrahedra.
The collection of all cells over $\mathbf{x}\in\mathcal{S}$ can be written as
$\mathcal{C} := \bigcup_{\mathbf{x} \in \mathcal{S}} \mathcal{C}_{\mathbf{x}}$,
which forms a polyhedral complex that partitions $\mathcal{S}$.
On each cell $\mathcal{C}_{\mathbf{x}}$, the encoder $\tau$ is affine. 
Polyhedral cell boundaries are where $\tau$ switches affine regions, forming fixed partitions that further refine the regions by the subsequent ReLU network. Unlike the adaptive partitions learned by the ReLU layers, those of the tetrahedral positional encoder are determined by hyperparameters and remain fixed during training.
\end{definition}

\begin{figure*}[t]
  \centering
  \includegraphics[width=1.0\textwidth]{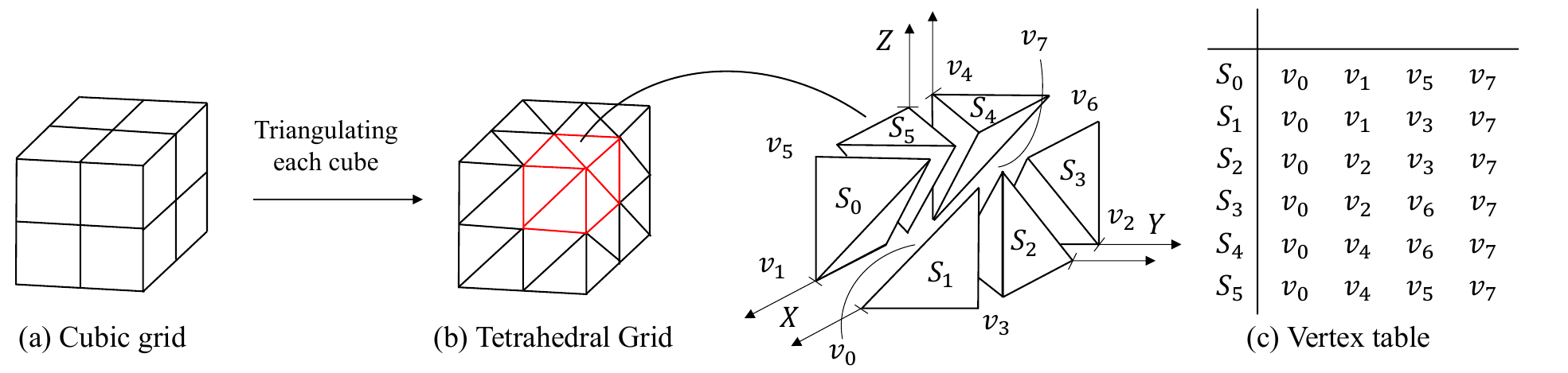}
  \caption{At each level, each cube cell is decomposed into six congruent tetrahedra.}
  \label{fig:tetsubdiv_mtd}
\end{figure*}

\subsection{Input Preconditioning}
\label{sec:precondition}
In our positional encoder, the barycentric coordinates $\mathbf{w}(\mathbf{x}) = [w_1(\mathbf{x}), w_2(\mathbf{x}), w_3(\mathbf{x})]^\top$ with $w_0(\mathbf{x}) = 1 - \sum_{i=1}^3 w_i(\mathbf{x})$ are an affine mapping of the input $\mathbf{x}$ inside each tetrahedron $\mathcal{T}$ (see Appendix~\ref{appx:theory}). 
Hence, there exist a constant matrix $\mathbf{J}_\mathcal{T} \in \mathbb{R}^{3 \times 3}$ and vector $\mathbf{b}_\mathcal{T}$ such that
\begin{equation}
\mathbf{w}(\mathbf{x}) = \mathbf{J}_\mathcal{T} \mathbf{x} + \mathbf{b}_\mathcal{T}, \quad \text{and} \quad 
d\mathbf{w} = \mathbf{J}_\mathcal{T}\, d\mathbf{x}.
\end{equation}
The Jacobian $\mathbf{J}_\mathcal{T}$ describes how barycentric weights vary \wrt spatial displacements, and therefore controls the direction-dependent sensitivity of feature updates during backpropagation.
Anisotropy in $\mathbf{J}_\mathcal{T}$ leads to uneven effective learning rates along different directions.
For the six-tetrahedra subdivision used in our encoder (\cref{fig:tetsubdiv_mtd}), the average local metric
$\mathbf{M} = \tfrac{1}{6}\sum_\mathcal{T} \mathbf{J}_\mathcal{T}^\top \mathbf{J}_\mathcal{T}$
is inherently anisotropic: directions orthogonal to the cube body diagonal are disproportionately amplified.
To mitigate this geometric bias, we introduce a global linear preconditioner $\mathbf{A}^*$ and feed $\mathbf{x}' = \mathbf{A}^* \mathbf{x}$ into the encoder.
This transformation isotropizes the average metric and reduces each local condition number from $16.39$ to $5.05$.
For simplicity, we reuse $\mathbf{x}$ to denote the preconditioned coordinates hereafter.
The explicit form and derivation of $\mathbf{A}^*$ are given in Appendix~\ref{supp:global_preconditioner}.

\subsection{Initial Skeleton Extraction from the Encoder}
\label{sec:grid-skeleton}

Naively enumerating all $\mathcal{C}_{\mathbf{x}}$ is prohibitive at multi-resolution scales, so we propose a tensorized algorithm that exploits shared plane normals with level-dependent offsets to extract the vertices and edges of $\mathcal{C}$ as an initial skeleton.
Since cell vertices are the intersections of tetrahedra across resolutions, we exploit the grid’s regularity by representing $\mathcal{C}$ as a union of parallel-plane sets.
The slopes of the tetrahedral subdivisions are constant across resolution levels, their normals are shared, while their offsets vary. 
Thus, multiple planes can be compactly represented as sharing a single normal but differing in offset. 
We denote the set of unique normals as
$\mathcal{N} = \{\mathbf{n}_0, \mathbf{n}_1, \dots, \mathbf{n}_i, \dots\}$,
and, for each normal $\mathbf{n}_i$, define the associated offsets across all resolution levels as
$\mathbf{d}^{(i)} = [\,d^{(i)}_0,\, d^{(i)}_1,\, \dots\,]^\top$.

Vertices are extracted as intersections of three or more planes with linearly independent normals, and edges as intersections of two or more non-parallel planes. 
Note that the tetrahedral subdivision in Fig.~\ref{fig:tetsubdiv_mtd} involves only six fixed plane normals, while their offsets vary across resolution levels. 
By grouping a small set of shared normals with their corresponding offsets, the extraction process can be efficiently implemented using parallel tensor operations. For the detailed algorithm, refer to Appendix~\ref{sec:polycomplex-construction}.

\subsection{Encoding Regions and Boundaries}
\label{sec:region-indicators}

Sign vectors allow us to identify the linear region of a preconditioned input $\mathbf{x}$ within the ReLU network $\nu^{(M)}$ and to determine whether $\mathbf{x}$ lies on a decision boundary.
However, the encoder $\tau$ further refines these regions according to the boundary of $\mathcal{C}_{\mathbf{x}}$.
To correctly represent regions under the full tetrahedral network $f = \nu^{(M)} \circ \tau$, we therefore require a more specialized indicator that also encodes the position of $\mathbf{x}$ within the grid structure induced by $\tau$.

\paragraph{Barycentric masks}
The barycentric coordinates of $\mathbf{x}$ in the tetrahedron $\mathcal{T}^{(\ell)}(\mathbf{x})$ at level~$\ell$ are defined as:
\[\mathbf{w}^{(\ell)}(\mathbf{x}) = [w_0^{(\ell)}(\mathbf{x}), w_1^{(\ell)}(\mathbf{x}), w_2^{(\ell)}(\mathbf{x}), w_3^{(\ell)}(\mathbf{x})].\]
Then, we define the barycentric mask at level~$\ell$ as:
\[
\mathbf{m}^{(\ell)}(\mathbf{x})
=
[m^{(\ell)}_0(\mathbf{x}),\,
 m^{(\ell)}_1(\mathbf{x}),\,
 m^{(\ell)}_2(\mathbf{x}),\,
 m^{(\ell)}_3(\mathbf{x})],
\]
where
\begin{equation}
m^{(\ell)}_i(\mathbf{x}) =
\begin{cases}
1, & w_i^{(\ell)}(\mathbf{x}) > \epsilon_b,\\[4pt]
0, & w_i^{(\ell)}(\mathbf{x}) \le \epsilon_b,
\end{cases}
\quad i \in \{0,1,2,3\},
\label{eq:bary-mask}
\end{equation}
and $\epsilon_b > 0$ is a small threshold.
A zero entry indicates that $\mathbf{x}$ lies on the boundary of $\mathcal{T}^{(\ell)}(\mathbf{x})$.
Accordingly, a single zero indicates that $\mathbf{x}$ lies on a face of the tetrahedron, two zeros indicate that it lies on an edge, and three zeros indicate that it lies on a vertex.
The \emph{barycentric masks} of $\mathbf{x}$ across all levels are concatenated to form:
\begin{equation}
\mathbf{m}(\mathbf{x}) = \bigoplus_{\ell=0}^{L-1} \mathbf{m}^{(\ell)}(\mathbf{x}) \in \{0,1\}^{4L}.
\end{equation}

Recall that $\mathcal{C}_{\mathbf{x}}$ is defined by the geometric intersection of the $\mathbf{x}$-containing tetrahedra $\mathcal{T}^{(\ell)}(\mathbf{x})$; thus the barycentric mask indicates on which boundary of the $\mathcal{C}_{\mathbf{x}}$ the point lies.
\paragraph{Region indicators}
To uniquely locate $\mathbf{x}$ within the $\mathcal{C}$, we introduce an additional spatial encoding, referred to as the \emph{region indicator}.
For a given level $\ell \in \{0, \dots, L-1\}$ with resolution $N_\ell$, 
the \emph{grid anchor} of a point $\mathbf{x}$ is defined as:
$\mathbf{a}^{(\ell)}(\mathbf{x}) := \big\lfloor N_\ell \mathbf{x} \big\rfloor$,
where $\lfloor \cdot \rfloor$ denotes the componentwise floor operator.
This anchor corresponds to the cube’s lexicographically smallest vertex 
(the ``$v_0$'' corner in \cref{fig:tetsubdiv_mtd}).
Each cube at level~$\ell$ is subdivided into six congruent tetrahedra.
The \emph{tetra offset} $t^{(\ell)}(\mathbf{x}) \in \{0,1,2,3,4,5\}$ indicates which of these tetrahedra anchored at $\mathbf{a}^{(\ell)}(\mathbf{x})$ contains $\mathbf{x}$, according to the subdivision scheme shown in \cref{fig:tetsubdiv_mtd}c.
We then define the \emph{tetra index} at level~$\ell$ as:
$\mathbf{r}^{(\ell)}(\mathbf{x}) :=
[\,\mathbf{a}^{(\ell)}(\mathbf{x}),\; t^{(\ell)}(\mathbf{x})\,] \in \mathbb{Z}^4$.
The \emph{region indicator} is obtained by concatenating all levelwise tetra indices:
\begin{equation}
\mathbf{r}(\mathbf{x}) :=
\bigoplus_{\ell=0}^{L-1} \mathbf{r}^{(\ell)}(\mathbf{x})
\in \mathbb{Z}^{4L}.
\end{equation}

This region indicator encodes the intersection of the tetrahedra containing $\mathbf{x}$ across all levels, serving as a unique spatial index of the cell $\mathcal{C}_{\mathbf{x}}$ within the $\mathcal{C}$.

\subsection{Identifying Neighboring Cells on the Grid}
\label{sec:perturbation}
First, for the grid part, we use both the region indicator $\mathbf{r}(\mathbf{x}) \in \mathbb{Z}^{4L}$ 
and the barycentric mask $\mathbf{m}(\mathbf{x}) \in \{0,1\}^{4L}$ 
to identify neighboring cells when the point $\mathbf{x}$ lies on the boundary of its current cell $\mathcal{C}_{\mathbf{x}}$.
For each resolution level~$\ell$, we consider the corresponding 
$\mathbf{r}^{(\ell)}(\mathbf{x})$ and $\mathbf{m}^{(\ell)}(\mathbf{x})$.

If one or more entries of $\mathbf{m}^{(\ell)}(\mathbf{x})$ are zero,
then $\mathbf{x}$ lies on the corresponding boundary of $\mathcal{T}^{(\ell)}(\mathbf{x})$.
To enumerate all neighboring cells, we perturb the tetra index $\mathbf{r}^{(\ell)}(\mathbf{x})$
according to the zero pattern in $\mathbf{m}^{(\ell)}(\mathbf{x})$.
For each level~$\ell$, we precompute a lookup table of neighbor configurations (see Appendix~\ref{supp:neighbor_lut}).
Each entry in this table consists of an anchor displacement
$\Delta_i \in \mathbb{Z}^3$ and a tetra offset $t_i \in \{0,\dots,5\}$,
corresponding to face, edge, or vertex adjacency.
The pair $(\Delta_0, t_0) = ([0,0,0], t^{(\ell)}(\mathbf{x}))$ represents
$\mathcal{T}^{(\ell)}(\mathbf{x})$ itself, while the remaining entries specify neighboring tetrahedra.
The tetra index of the $i$-th neighbor at level~$\ell$ is then given by
$\mathbf{r}^{(\ell)}_i(\mathbf{x})
= \big[\mathbf{a}^{(\ell)}(\mathbf{x}) + \Delta_i,\; t_i\big]$.

To obtain the complete set of neighboring polyhedral cells, 
the levelwise tetra indices $\mathbf{r}^{(\ell)}_i(\mathbf{x})$ are concatenated across all levels for every possible combination of neighbors. 
This expansion yields $N_{\text{grid}} = K_0 K_1 \dots K_{L-1}$ region indicators where $K_\ell$ denotes the number of neighboring tetrahedra including itself.
The region indicator corresponding to the $n$-th neighboring cell for $\mathcal{C}_\mathbf{x}$ is:
\begin{equation}
\mathbf{r}_n(\mathbf{x}) 
= \mathbf{r}^{(0)}_{p}(\mathbf{x})\oplus\,
       \mathbf{r}^{(1)}_{q}(\mathbf{x})\oplus\, 
       \cdots\,\oplus
       \,\mathbf{r}^{(L-1)}_{r}(\mathbf{x})
\end{equation}
where each $\mathbf{r}^{(\ell)}_{j}(\mathbf{x})$ denotes the tetra index of the $j$-th neighboring tetrahedron at level~$\ell$. The collection of all such indices defines the set of grid-based neighboring regions:
\begin{equation}
\mathcal{A}_{\text{grid}}(\mathbf{x})
= 
\{\,\mathbf{r}_0(\mathbf{x}),\, \mathbf{r}_1(\mathbf{x}),\, \cdots,\, \mathbf{r}_{N_{\text{grid}}-1}(\mathbf{x})\,\}.
\end{equation}

Second, for the ReLU MLP part, we follow the same procedure described for $\nu^{(m)}_k$ in \cref{sec:edge-subdivision}. 
The sign-vector $\mathbf{s}^{(m)}_k(\mathbf{x})$ is perturbed at its zero entries, 
producing all adjacent regions of ReLU MLP around the boundary on which $\mathbf{x}$ lies. 
The set of perturbed sign-vectors is denoted by:
\begin{equation}
\mathcal{A}_{\text{relu}}(\mathbf{x})
=
\{\,\mathbf{s}^{(m)}_{(k,0)}(\mathbf{x}),\, \cdots,\, \mathbf{s}^{(m)}_{(k,{N_{\text{relu}}-1})}(\mathbf{x})\,\},
\end{equation}
where $\mathbf{s}^{(m)}_{(k,i)}(\mathbf{x})$ denotes the sign-vector of the $i$-th adjacent region, 
including the one containing $\mathbf{x}$ itself. 
Here, $N_{\text{relu}} = 2^{n_0}$, with $n_0$ denoting the number of zero entries in $\mathbf{s}^{(m)}_k(\mathbf{x})$.

Finally, we combine the grid-based and ReLU-based neighbor sets 
to obtain the complete set of parent regions adjacent to the query point~$\mathbf{x}$:
$
\mathcal{A}(\mathbf{x}) 
=
\big\{\,\mathbf{r}_i(\mathbf{x})\,\oplus\, \mathbf{s}^{(m)}_{(k,j)}(\mathbf{x})\, 
\}
$
where $\mathbf{r}_i(\mathbf{x}) \in \mathcal{A}_{\text{grid}}(\mathbf{x})$ and
$\mathbf{s}^{(m)}_{(k,j)}(\mathbf{x}) \in \mathcal{A}_{\text{relu}}(\mathbf{x})
$.
For two query points $\mathbf{x}_a$ and $\mathbf{x}_b$,
comparing $\mathcal{A}(\mathbf{x}_a)$ and $\mathcal{A}(\mathbf{x}_b)$ provides a test for region equivalence: if they share a common element, two points lie within or on the boundary of the same region.

\subsection{Face Extraction}
\label{sec:face-extraction}

After the edge subdivision (\cref{sec:edge-subdivision}) using the sign vectors (\cref{sec:sign-vectors}) and the region indicator (\cref{sec:region-indicators}), we obtain the candidate vertex set $\mathcal{V}$ and the edge set $\mathcal{E}$. We then select the zero-level subsets from these as follows:
$\mathcal{V}^\ast = \{\mathbf{x} \in \mathcal{V} \mid |f(\mathbf{x})| \le \epsilon_f\},\quad 
 \mathcal{E}^\ast = \{(\mathbf{x}_a, \mathbf{x}_b) \in \mathcal{E} \mid \mathbf{x}_a, \mathbf{x}_b \in \mathcal{V}^\ast\},$ where $\epsilon_f$ is a small threshold used to extract the zero-level set of $f$.
The final triangle mesh is obtained through a straightforward procedure that connects adjacent vertices according to their local connectivity~\cite{Kim2024}.

\section{Experiments}
\label{sec:exp}
\noindent\textbf{Datasets}~~
We evaluate on three datasets: the Stanford 3D Scanning Repository with 5 meshes~\cite{curless1996volumetric}, ABC with 100 randomly selected meshes~\cite{Koch_2019_CVPR}, and Thingi10K with 500 randomly selected closed meshes~\cite{Thingi10K}.
We follow the preprocessing of DeepSDF normalization~\cite{Park2019DeepSDF}.
For each ground-truth mesh, we sample $500$K SDF query points, drawing the majority from a narrow band around the surface and the remainder uniformly within the bounding volume.

\begin{figure*}[t]
  \centering
  \includegraphics[width=1.0\linewidth]{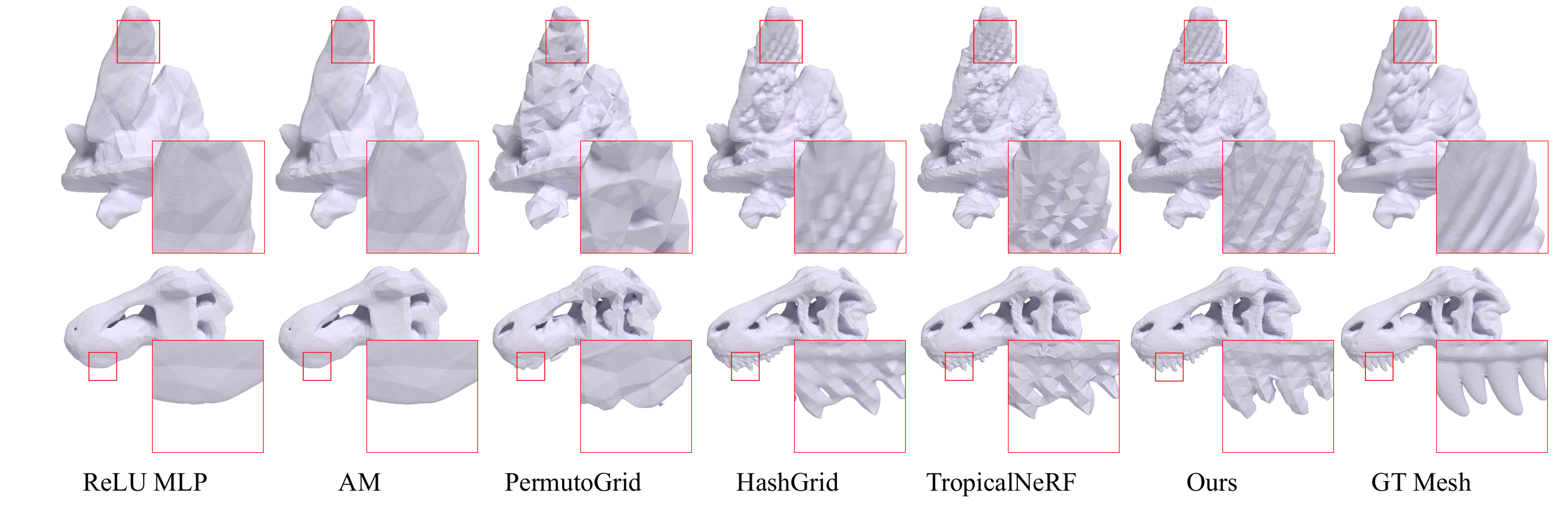}
  \caption{Qualitative comparison on Thingi10K corresponding to \cref{tab:baselines_cd}.}
  \label{fig:baselines}
\end{figure*}

\noindent\textbf{Settings}~~
We use a multi-resolution tetrahedral positional encoder with $L=4$ levels and feature dimension $d=2$ per level, followed by a ReLU MLP with three hidden layers of width 12 each.
We consider three resolution settings: \emph{Small} $(N_{\min},N_{\max})=(2,32)$, \emph{Medium} $(4,64)$, and \emph{Large} $(8,128)$.
We train networks for 10 epochs using an $\ell_1$ SDF loss and an eikonal regularizer with weight $\lambda_{\mathrm{eik}} = 5 \times 10^{-3}$.
All experiments use a single NVIDIA V100 GPU; training a network for each SDF takes about 1 minute.

\noindent\textbf{Baselines}~~
We consider both network architectures for SDF learning and meshing algorithms for isosurface extraction.
For network architectures, we use grid-based positional encoders—HashGrid~\cite{mueller2022instant} and PermutoGrid~\cite{rosu2023permutosdf}—each followed by a ReLU MLP with the same depth, width, and resolution settings.
We also include a plain ReLU MLP configured as in Analytic Marching (AM)~\cite{Lei2020AnalyticMarching}, with ten hidden layers of width~90. Unless otherwise noted, meshes for these networks are extracted using Marching Cubes at $512^3$ resolution.
We apply the corresponding analytic methods: the AM extractor to the plain ReLU MLP and TropicalNeRF~\cite{Kim2024} to the HashGrid. In both cases, meshes are extracted from the same corresponding trained networks used for SDF evaluation.
For sampling-based methods, we compare against the Marching Cubes (MC), Marching Tetrahedra (MT), and Dual Contouring (DC). We use publicly released implementations for HashGrid, PermutoGrid, AM, and TropicalNeRF.

\noindent\textbf{Evaluation metrics}~~
We measure geometric accuracy with Chamfer Distance (CD), computed from $10^6$ surface samples between the extracted mesh and the ground-truth (GT) mesh.
We also assess self-consistency—agreement between the extracted mesh and the network’s isosurface—using three additional metrics.
For Surface-sampled SDF (SSDF) and Angular Difference (AD), we uniformly sample $10^6$ points on the surface of the extracted mesh and query the network: SSDF is the mean absolute SDF value predicted by the network, while AD is the mean absolute angle between the mesh normals and the network-derived normals.
Vertex-sampled SDF (VSDF) is SSDF computed at the extracted mesh’s vertices.
In all three metrics, zero indicates perfect self-consistency, \ie, the extracted mesh exactly matches the network’s zero-level set.

\subsection{Ground-truth Accuracy}
\label{sec:gt_accuracy}
\begin{table}[t]
\centering
\footnotesize
\setlength{\tabcolsep}{2.6pt}
\renewcommand{\arraystretch}{0.95}

\begin{minipage}[t]{0.49\columnwidth}
\centering
\setlength{\abovecaptionskip}{0pt}
\setlength{\belowcaptionskip}{0pt}

\captionof{table}{CD ($\times 10^{-6}$) to ground-truth meshes on Thingi10K, ABC, Stanford at (\emph{Large}) resolution setting.}
\label{tab:baselines_cd}
\resizebox{0.93\linewidth}{!}{%
\begin{tabular}{lccc}
\toprule
Method & Stanford & ABC & Thingi10K \\
\midrule
ReLU MLP      & 5480 & 4584 & 3779 \\
AM            & 5475 & 4570 & 3775 \\
PermutoGrid   & 3644 & 3104 & 2897 \\
HashGrid      & \textbf{1659} & 1854 & 1763 \\
TropicalNeRF  & 1737 & 1866 & 1809 \\
\midrule
Ours          & 1718 & \textbf{1758} & \textbf{1722}\\
\bottomrule
\end{tabular}%
}
\end{minipage}\hfill
\begin{minipage}[t]{0.49\columnwidth}
\centering
\setlength{\abovecaptionskip}{0pt}
\setlength{\belowcaptionskip}{0pt}

\captionof{table}{CD ($\times 10^{-6}$) to ground-truth meshes on Stanford with analytic extraction baselines across resolutions.}
\label{tab:different_resolution}
\resizebox{\linewidth}{!}{%
\begin{tabular}{clcccccc}
\toprule
$R$ & Method & Bunny & Dragon & Happy & Arma. & Lucy & Avg. \\
\midrule
-- & AM            & 3628 & 6021 & 6839 & 4577 & 6309 & 5475 \\
\midrule
S  & Tropical       & 4093 & 5674 & 7187 & \textbf{3671} & 6208 & 5367 \\
   & Ours           & \textbf{2810} & \textbf{4160} & \textbf{4889} & 3794 & \textbf{5398} & \textbf{4210} \\
\midrule
M  & Tropical       & 2201 & 2606 & 2984 & 2506 & 3555 & 2770 \\
   & Ours           & \textbf{2024} & \textbf{2483} & \textbf{2362} & \textbf{2300} & \textbf{3141} & \textbf{2462} \\
\midrule
L  & Tropical       & 1842 & \textbf{1726} & 1792 & 1611 & 1714 & 1737 \\
   & Ours           & \textbf{1817} & 1732 & \textbf{1779} & \textbf{1581} & \textbf{1681} & \textbf{1718} \\
\bottomrule
\end{tabular}%
}
\end{minipage}

\end{table}
We report CD to ground-truth meshes across network baselines in~\cref{tab:baselines_cd} under the \emph{Large} resolution setting. Our method attains the lowest CD on Thingi10K and ABC among all baselines and surpasses all analytic-extraction methods across datasets.
\cref{fig:baselines} provides qualitative comparisons on Thingi10K for~\cref{tab:baselines_cd}.
A plain ReLU MLP shows limited ability to represent complex SDFs.
Moreover, HashGrid may smooth high-curvature regions because of the implicit smoothing of trilinear interpolation, whereas our method better preserves sharp features.
\begin{table}[t]
\centering
\setlength{\abovecaptionskip}{0pt}
\setlength{\belowcaptionskip}{0pt}
\scriptsize
\setlength{\tabcolsep}{2.2pt}
\renewcommand{\arraystretch}{1.02}
\captionof{table}{GT mesh-sampled SDF MAE ($\times 10^{-3}$, $\downarrow$).
We sample $10^6$ points uniformly on the GT mesh surface, query the network, and report the mean absolute SDF value.}
\label{gt_sdf}
\begin{tabular}{lccc ccc ccc}
\toprule
\multirow{2}{*}{Method} &
\multicolumn{3}{c}{Small} &
\multicolumn{3}{c}{Medium} &
\multicolumn{3}{c}{Large} \\
& Stanford & ABC & Thingi10K & Stanford & ABC & Thingi10K & Stanford & ABC & Thingi10K \\
\midrule
PermutoGrid & 7.62 & 4.54 & 5.36 & 4.34 & 2.66 & 3.05 & 2.24 & 1.53 & 1.50 \\
HashGrid    & 3.24 & 2.19 & 2.67 & \textbf{1.36} & 1.21 & 1.12 & \textbf{0.53} & \textbf{0.61} & 0.50 \\
\midrule
Ours        & \textbf{2.82} & \textbf{1.74} & \textbf{1.79} & 1.39 & \textbf{1.12} & \textbf{1.06} & 0.62 & \textbf{0.61} & \textbf{0.49} \\
\bottomrule
\end{tabular}
\end{table}
We also compare analytic isosurface extraction baselines under different resolution settings on the Stanford dataset in \cref{tab:different_resolution} and \cref{fig:ablation_res}, reporting CD ($\times 10^{-6}$) against ground-truth meshes.
As the resolution decreases, TropicalNeRF’s accuracy degrades markedly, suggesting that its geometric heuristics struggle with the inherent nonlinearity of trilinear interpolation, whereas our method avoids these issues due to our CPWA structure and analytic extraction.
The qualitative results in \cref{fig:ablation_res} show that—even under the \emph{Small} setting—our method remains visually reasonable.

To decouple representation fidelity from extraction faithfulness, we report GT-mesh-sampled SDF mean absolute error (MAE) in \cref{gt_sdf}, which evaluates the learned field independently of the extractor.
Our GT-mesh SDF MAE is comparable to or better than grid-based encoder baselines across datasets and settings, indicating that our encoder design does not sacrifice SDF fitting fidelity while enabling exact isosurface recovery.
\cref{tab:ablation_a} reports the quantitative effect of the input preconditioner $\mathbf{A}^*$ on Thingi10K, with CD ($\times 10^{-6}$) measured against the ground-truth mesh under the \emph{Large} resolution setting, and \cref{fig:bunny_wA} shows qualitative results on the Stanford Bunny.
Both indicate improved SDF accuracy for our network with $\mathbf{A}^*$. Applying our encoder-derived $\mathbf{A}^*$ reduces CD for our method but worsens it for HashGrid, suggesting that $\mathbf{A}^*$ is tailored to our encoder rather than a generic preconditioner.
Under matched capacity—same levels, per-level table size, and resolution—the total number of learnable features is identical to HashGrid.
See Appendix~\ref{sec:appendix_Astar_hashgrid} for further analysis.

\begin{table*}[t]
\small
\centering
\setlength{\abovecaptionskip}{0pt}
\setlength{\belowcaptionskip}{0pt}
\captionof{table}{
Self-consistency of different meshing methods on Thingi10K, ABC, and Stanford:
SSDF and VSDF ($\times 10^{-6}$) and AD ($^\circ$) between each extracted mesh and its SDF network.
Marching Cubes at each resolution is applied to our trained network.
}
\label{tab:baselines_consistency}
\begin{tabular}{lccccccccc}
\toprule
\multirow{2}{*}{Method} &
\multicolumn{3}{c}{Thingi10K} &
\multicolumn{3}{c}{ABC} &
\multicolumn{3}{c}{Stanford} \\
\cmidrule(lr){2-4}\cmidrule(lr){5-7}\cmidrule(lr){8-10}
& SSDF~$\downarrow$ & VSDF~$\downarrow$ & AD~$\downarrow$
& SSDF~$\downarrow$ & VSDF~$\downarrow$ & AD~$\downarrow$
& SSDF~$\downarrow$ & VSDF~$\downarrow$ & AD~$\downarrow$ \\
\midrule
MC256       & 217.0 & 98.0 & 5.44 & 218.1 & 101.8 & 5.70 & 289.3 & 128.3 & 8.91 \\
MC512       & 65.2 & 27.0 & 3.10 & 66.7 & 29.2 & 3.23 & 100.1 & 40.2 & 5.81 \\
MC1024      & 17.6 & 7.79 & 1.68 & 18.9 & 8.26 & 1.72 & 20.5 & 6.00 & 2.19 \\
AM & 0.020 & 0.020 & 0.00 &
                         0.021 & 0.022 & 0.00 &
                         0.025 & 0.024 & 0.00 \\
TropicalNeRF           & 144.1 & 8.7 & 3.45 & 136.1 & 12.3 & 3.64 & 189.0 & 13.4 & 4.80 \\
\midrule
Ours                   & 0.078 & 0.082 & 0.00 &
                         0.077 & 0.080 & 0.00 &
                         0.076 & 0.078 & 0.00 \\
\bottomrule
\end{tabular}
\end{table*}

\begin{figure*}[t]
\centering
\footnotesize

\captionsetup[table]{skip=5pt}
\captionsetup[figure]{skip=6pt}
\setlength{\abovecaptionskip}{0pt}
\setlength{\belowcaptionskip}{0pt}

\begin{minipage}[t]{0.49\textwidth}
\centering

\small
\centering
\captionof{table}{CD ($\times 10^{-6}$) on the Stanford dataset for sampling-based extractors applied to our trained network under different resolution schedules, compared against pseudo ground-truth per method.}
\label{tab:mcmtdc}
\resizebox{\linewidth}{!}{%
\begin{tabular}{l rc rc rc}
\toprule
\multirow{2}*{Method} &
\multicolumn{2}{c}{Small} &
\multicolumn{2}{c}{Medium} &
\multicolumn{2}{c}{Large} \\
\cmidrule(lr){2-7}
 & {$|\mathcal{V}|\downarrow$} & {CD$\downarrow$}
 & {$|\mathcal{V}|\downarrow$} & {CD$\downarrow$}
 & {$|\mathcal{V}|\downarrow$} & {CD$\downarrow$} \\
\midrule
MC128  & 24,768  & 1734 & 24,735  & 1821 & 25,646  & 1871 \\
MC256  & 101,337 & 1355 & 103,444 & 1380 & 105,227 & 1418 \\
MC512  & 409,321 & 1290 & 418,464 & 1308 & 426,394 & 1323 \\
\midrule
Ours   & 29,557  & \textbf{1286} & 73,851 & \textbf{1301} & 280,929 & \textbf{1316} \\
\midrule[\heavyrulewidth]
MT128  & 96,890  & 1677 & 88,928 & 1740 & 90,116 & 1806 \\
MT256  & 396,516 & 1397 & 364,464 & 1363 & 370,935 & 1387 \\
\midrule
Ours   & 29,557  & \textbf{1368} & 73,851 & \textbf{1314} & 280,929 & \textbf{1332} \\
\midrule[\heavyrulewidth]
DC128  & 41,461  & 2099 & 41,794 & 1524 & 42,048 & 1735 \\
DC256  & 167,873 & 1351 & 169,164 & 1352 & 170,670 & 1437 \\
DC512  & 676,567 & \textbf{1329} & 682,080 & 1334 & 688,562 & 1341 \\
\midrule
Ours   & 29,557  & 1348 & 73,851 & \textbf{1331} & 280,929 & \textbf{1336} \\
\bottomrule
\end{tabular}%
}

\vfill
\vspace{1em}
\captionof{table}{CD ($\times 10^{-6}$) with/without $\mathbf{A}^*$ on Thingi10K;
$\Delta=\mathrm{CD}_{\mathrm{w/o}}-\mathrm{CD}_{\mathrm{w/}}$ (positive is better).}
\label{tab:ablation_a}

\begin{minipage}[b]{0.65\linewidth}
\centering
\setlength{\tabcolsep}{2.6pt}
\renewcommand{\arraystretch}{0.95}

\resizebox{\linewidth}{!}{%
\begin{tabular}{lccc}
\toprule
Method & w/o~$\mathbf{A}^*$ & w/~$\mathbf{A}^*$ & $\Delta \uparrow$ \\
\midrule
HashGrid & 1763 & 1993 & $-230$ \\
Ours     & 1856 & 1722 & \textbf{134} \\
\bottomrule
\end{tabular}%
}
\end{minipage}

\end{minipage}
\hfill
\begin{minipage}[t]{0.49\textwidth}
\vspace{0pt}
\centering

\includegraphics[
  width=0.9\linewidth,
  trim=0mm 2mm 0mm 2mm, clip
]{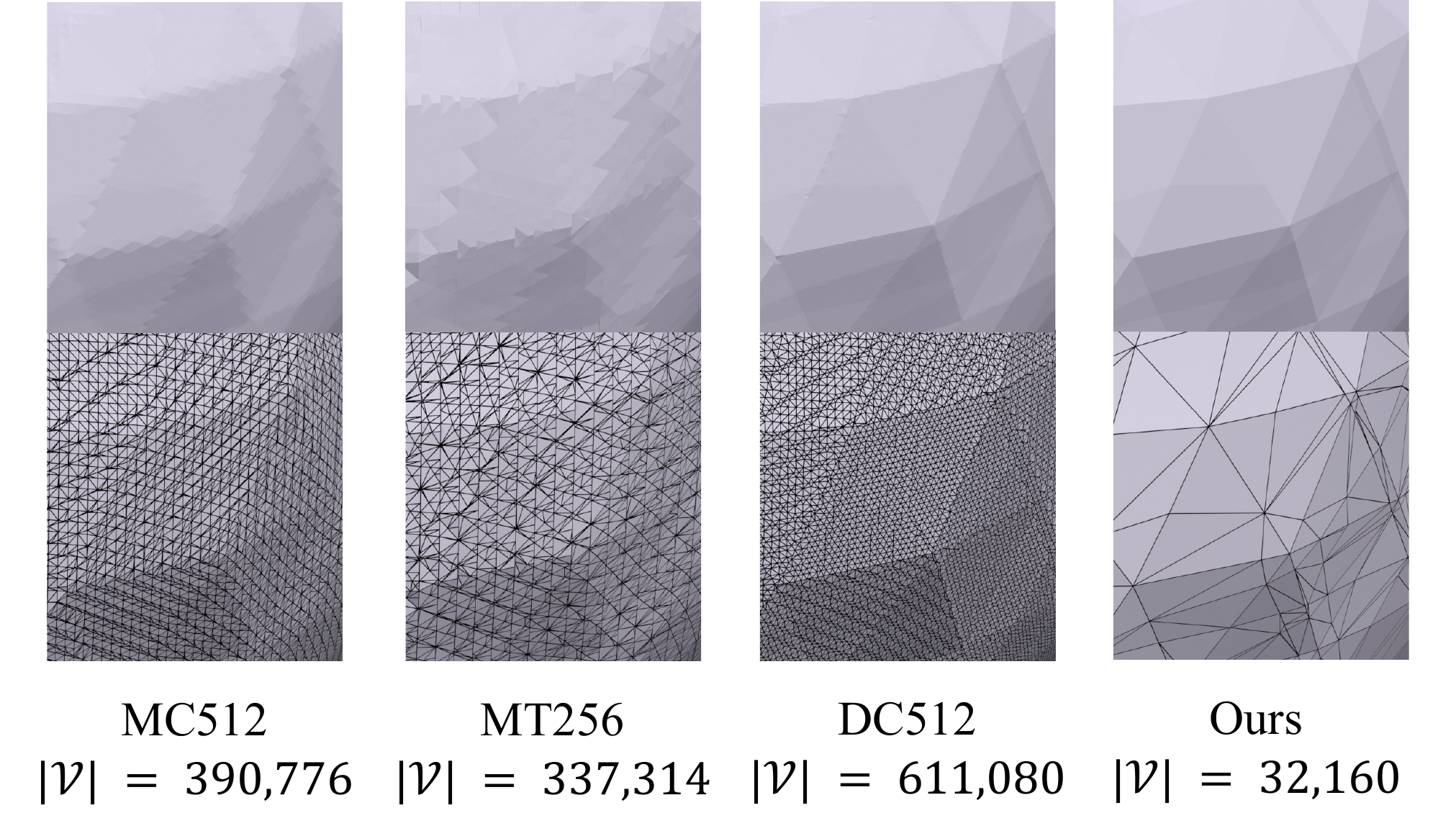}
\captionof{figure}{Qualitative comparison under the \emph{Small} setting. Sampling-based methods often over-fragment triangles with surface artifacts, while our method avoids both by construction.}
\label{fig:mcmtdc}

\vfill
\vspace{1em}

\begin{minipage}[b]{\linewidth}
\centering
\includegraphics[
  width=0.9\linewidth,
  trim=0mm 2mm 0mm 2mm, clip
]{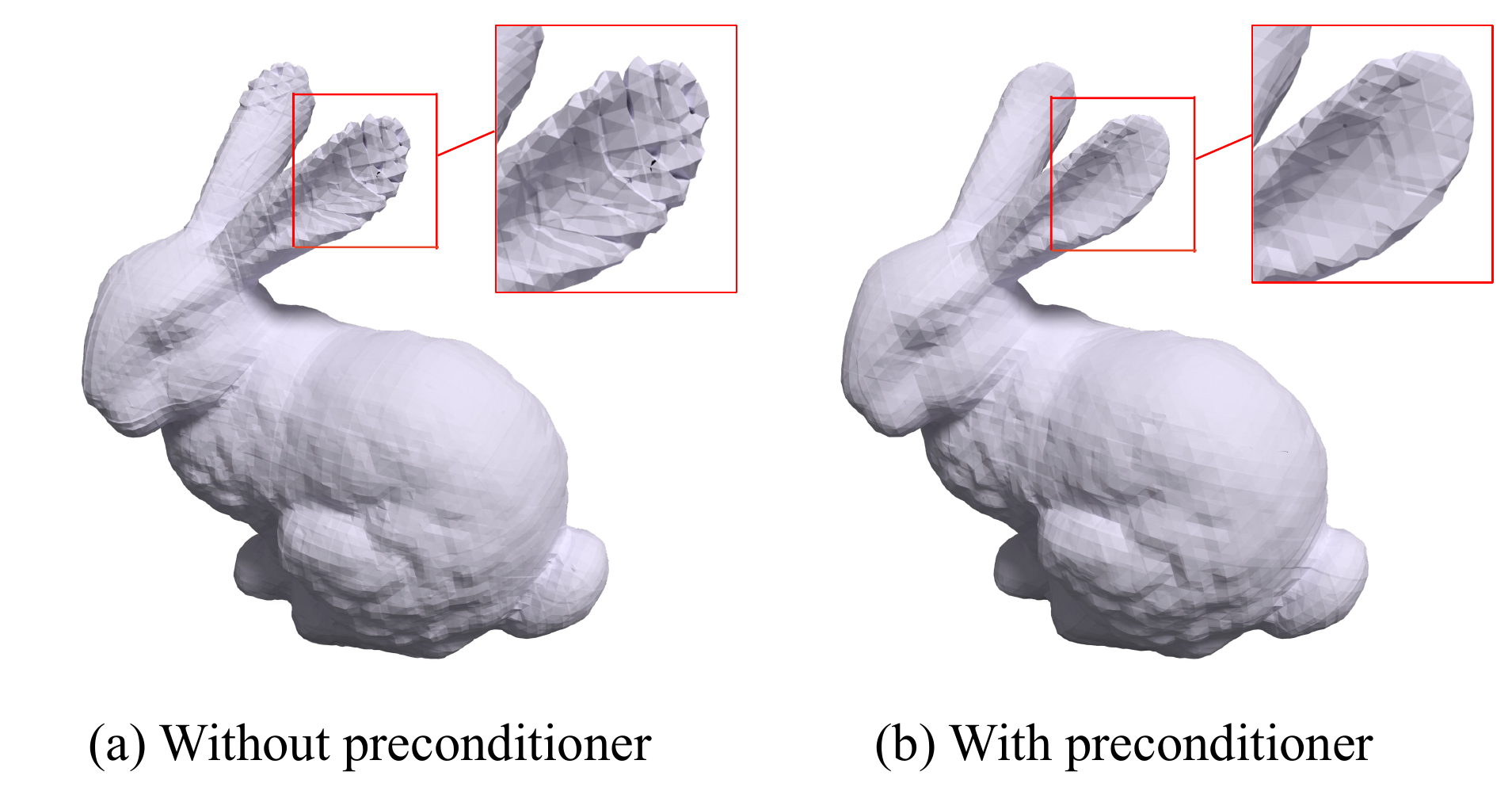}
\captionof{figure}{Qualitative effect of the input preconditioner $\mathbf{A}^*$ on the Stanford Bunny. The preconditioner improves the network’s SDF accuracy, especially in high-curvature regions (\eg, the ears).}
\label{fig:bunny_wA}
\end{minipage}

\end{minipage}

\end{figure*}
\subsection{Self-consistency with the Network’s Isosurface}
We evaluate self-consistency in \cref{tab:baselines_consistency} using SSDF ($\times 10^{-6}$), VSDF ($\times 10^{-6}$), and AD (\textdegree).
Analytic baselines are measured on their own trained networks with their corresponding extractors, whereas MC256, MC512, and MC1024 are applied to our trained network.
Both AM and our method achieve both SSDF and VSDF on the order of $10^{-8}$, and AD rounds to 0.00.
These values are near machine precision, yielding $\sim10^{3}\times$ stronger self-consistency than MC1024.
By contrast, TropicalNeRF shows lower self-consistency, consistent with the limits of its geometric heuristics for handling trilinear interpolation.
The small discrepancy between AM and our method stems from running our encoder in single precision and reflects machine-precision effects, including tie-breaking at tetrahedral grid boundaries and barycentric weight evaluation.


A broader comparison with widely used sampling-based extractors is given in \cref{tab:mcmtdc} and \cref{fig:mcmtdc}.
\cref{tab:mcmtdc} reports CD (\(\times 10^{-6}\)) for MC, MT, and DC across multiple grid resolutions and across our network’s different resolution settings with all meshes extracted from our trained network.
For each method, we compare against a high-resolution pseudo ground-truth—MC1024 for MC, MT512 for MT, and DC1024 for DC—chosen to closely approximate the network’s isosurface.
In \cref{tab:mcmtdc}, our method achieves stronger consistency with fewer vertices across extractor resolutions and network settings, except for DC512 in the \emph{Small} setting.
This exception is consistent with DC’s use of network normals, whereas MC and MT use only SDF values.
Still, DC’s slight edge comes at the cost of roughly \(30\times\) more vertices than ours.
As shown in \cref{fig:mcmtdc}, sampling-based methods often need a large number of triangles to improve self-consistency, while discrete sampling itself introduces staircase artifacts.
In contrast, ours achieves better self-consistency at a fixed, low triangle budget with $10\times$--$20\times$ fewer vertices, and avoids over-fragmentation and staircase artifacts.
For runtime and memory, and for additional comparisons with adaptive sampling-based extractors and AM applied to larger ReLU MLPs, see Appendix~\ref{appendix:add_results}.

\section{Conclusion}
We introduced \textit{TetraSDF}, an analytic isosurface extraction framework that preserves the CPWA structure of learned SDFs through a multi-resolution tetrahedral positional encoder with barycentric interpolation. This encoder generates an explicitly indexed polyhedral complex, enabling analytic extraction via region indicators, barycentric masks, and a grid-aware edge subdivision scheme that jointly tracks polyhedral cells and ReLU linear regions.
We further introduced a closed-form input preconditioner derived from the encoder-induced metric to reduce directional bias and improve training stability.
Across multiple benchmarks, we match or surpass grid-based encoder baselines in SDF accuracy while producing highly self-consistent meshes faithful to the learned isosurfaces, enabled by an efficient, GPU-parallel tensorized formulation.

Future work includes integrating TetraSDF into end-to-end neural reconstruction pipelines. In such settings, our field-faithful isosurface extraction could provide a consistent mesh-based supervision signal via field–mesh agreement under compact triangle budgets, without resolution-driven over-fragmentation.

\section{Acknowledgements}
The NAVER Smart Machine Learning (NSML) platform~\cite{kim2018nsml} was used for the experiments.
This work was supported by NAVER (Neural 3D Representation and Asset Generation for Web Rendering) and the National Research Foundation of Korea (RS-2026-25497603: Fundamentals of Online 4D Reconstruction with Semantic Viewpoint Control and Generative Priors for Proactive Content Consumption).

%
%
\bibliographystyle{splncs04}
\bibliography{main} 

@String(CVPR  = {IEEE Conf. Comput. Vis. Pattern Recog.})

@String(ICCV  = {Int. Conf. Comput. Vis.})

@String(ECCV  = {Eur. Conf. Comput. Vis.})

@String(NeurIPS = {Adv. Neural Inform. Process. Syst.})

@String(ICML  = {Int. Conf. Mach. Learn.})

@String(ICLR  = {Int. Conf. Learn. Represent.})

@String(CVPR  = {CVPR})

@String(ICCV  = {ICCV})

@String(ECCV  = {ECCV})

@String(NeurIPS = {NeurIPS})

@String(ICML  = {ICML})

@String(ICLR  = {ICLR})

@inproceedings{berzins2023polyhedral,
  title={Polyhedral complex extraction from {ReLU} networks using edge subdivision},
  author={Berzins, Arturs},
  booktitle={International Conference on Machine Learning (ICML)},
  pages={2234--2244},
  year={2023},
  organization={PMLR}
}

@book{boyd2004convex,
  title     = {Convex Optimization},
  author    = {Boyd, Stephen and Vandenberghe, Lieven},
  year      = {2004},
  publisher = {Cambridge University Press},
  address   = {Cambridge, UK},
  isbn      = {978-0521833783}
}

@article{burstedde2016tetrahedral,
  title={A tetrahedral space-filling curve for nonconforming adaptive meshes},
  author={Burstedde, Carsten and Holke, Johannes},
  journal={SIAM Journal on Scientific Computing},
  volume={38},
  number={5},
  pages={C471--C503},
  year={2016},
  publisher={SIAM}
}

@inproceedings{curless1996volumetric,
  title        = {A Volumetric Method for Building Complex Models from Range Images},
  author       = {Brian Curless and Marc Levoy},
  booktitle    = {Proceedings of the 23rd Annual Conference on Computer Graphics and Interactive Techniques (SIGGRAPH)},
  year         = {1996},
  pages        = {303--312},
  doi          = {10.1145/237170.237269}
}

@article{grigsby2022transversality,
  title={On transversality of bent hyperplane arrangements and the topological expressiveness of {ReLU} neural networks},
  author={Grigsby, J Elisenda and Lindsey, Kathryn},
  journal={SIAM Journal on Applied Algebra and Geometry},
  volume={6},
  number={2},
  pages={216--242},
  year={2022},
  publisher={SIAM}
}

@inproceedings{Hanin2019Complexity,
  author    = {Boris Hanin and David Rolnick},
  title     = {Complexity of Linear Regions in Deep Networks},
  booktitle = {Proceedings of the 36th International Conference on Machine Learning (ICML)},
  year      = {2019}
}

@inproceedings{Ju2002DualContouring,
  author    = {Tao Ju and Frank Losasso and Scott Schaefer and Joe Warren},
  title     = {Dual Contouring of Hermite Data},
  booktitle = {Proceedings of SIGGRAPH},
  year      = {2002}
}

@inproceedings{Kim2024,
  author = {Kim, Jin-Hwa},
  booktitle={Advances in Neural Information Processing Systems (NeurIPS)},
  volume={37},
  title = {Polyhedral Complex Derivation from Piecewise Trilinear Networks},
  year = {2024}
}

@InProceedings{Koch_2019_CVPR,
author = {Koch, Sebastian and Matveev, Albert and Jiang, Zhongshi and Williams, Francis and Artemov, Alexey and Burnaev, Evgeny and Alexa, Marc and Zorin, Denis and Panozzo, Daniele},
title = {{ABC}: A Big {CAD} Model Dataset For Geometric Deep Learning},
booktitle={Proceedings of the IEEE/CVF Conference on Computer Vision and Pattern Recognition (CVPR)},
year = {2019}
}

@inproceedings{
    Lei2020AnalyticMarching,
    title = {Analytic Marching: An Analytic Meshing Solution from Deep Implicit Surface Networks},
    author = {Jiabao Lei and Kui Jia},
    booktitle={International Conference on Machine Learning (ICML)},
    year = {2020}
}

@article{Lorensen1987MarchingCubes,
  author    = {William E. Lorensen and Harvey E. Cline},
  title     = {Marching Cubes: A High Resolution {3D} Surface Construction Algorithm},
  journal   = {ACM SIGGRAPH Computer Graphics},
  year      = {1987},
  volume    = {21},
  number    = {4},
  pages     = {163--169}
}

@inproceedings{Montufar2014LinearRegions,
  author    = {Guido F. Mont{\'u}far and Razvan Pascanu and Kyunghyun Cho and Yoshua Bengio},
  title     = {On the Number of Linear Regions of Deep Neural Networks},
  booktitle = {Advances in Neural Information Processing Systems (NeurIPS)},
  year      = {2014}
}

@article{mueller2022instant,
    author = {Thomas M\"uller and Alex Evans and Christoph Schied and Alexander Keller},
    title = {Instant Neural Graphics Primitives with a Multiresolution Hash Encoding},
    journal = {ACM Trans. Graph.},
    issue_date = {July 2022},
    volume = {41},
    number = {4},
    year = {2022},
    pages = {102:1--102:15},
    articleno = {102},
    numpages = {15},
    doi = {10.1145/3528223.3530127},
    publisher = {ACM},
    address = {New York, NY, USA},
}

@inproceedings{Raghu2017ExpressivePower,
  author    = {Maithra Raghu and Ben Poole and Jon Kleinberg and Surya Ganguli and Jascha Sohl{-}Dickstein},
  title     = {On the Expressive Power of Deep Neural Networks},
  booktitle = {Proceedings of the 34th International Conference on Machine Learning (ICML)},
  year      = {2017}
}

@inproceedings{rahaman2019spectral,
  title={On the spectral bias of neural networks},
  author={Rahaman, Nasim and Baratin, Aristide and Arpit, Devansh and Draxler, Felix and Lin, Min and Hamprecht, Fred and Bengio, Yoshua and Courville, Aaron},
  booktitle={International Conference on Machine Learning (ICML)},
  pages={5301--5310},
  year={2019},
  organization={PMLR}
}

@inproceedings{Serra2018Bounds,
  author    = {Thiago Serra and Christian Tjandraatmadja and Srikumar Ramalingam},
  title     = {Bounding and Counting Linear Regions of Deep Neural Networks},
  booktitle = {Proceedings of the 35th International Conference on Machine Learning (ICML)},
  year      = {2018}
}

@inproceedings{Tancik2020FourierFeatures,
  author    = {Matthew Tancik and Pratul P. Srinivasan and Ben Mildenhall and Sara Fridovich{-}Keil and Nithin Raghavan and Utkarsh Singhal and Ravi Ramamoorthi and Jonathan T. Barron and Ren Ng},
  title     = {Fourier Features Let Networks Learn High Frequency Functions in Low Dimensional Domains},
  booktitle = {Advances in Neural Information Processing Systems (NeurIPS)},
  year      = {2020}
}

@inproceedings{teschner2003optimized,
  title={Optimized spatial hashing for collision detection of deformable objects},
  author={Teschner, Matthias and Heidelberger, Bruno and M{\"u}ller, Matthias and Pomerantes, Danat and Gross, Markus H},
  booktitle={Vision, Modeling, and Visualization (VMV)},
  volume={3},
  pages={47--54},
  year={2003}
}

@article{Treece1999RMT,
  author    = {G. M. Treece and R. W. Prager and A. H. Gee},
  title     = {Regularised Marching Tetrahedra: Improved Iso-surface Extraction},
  journal   = {Computers \& Graphics},
  year      = {1999},
  volume    = {23},
  number    = {4},
  pages     = {583--598}
}

@inproceedings{Park2019DeepSDF,
  author    = {Jeong Joon Park and Peter Florence and Julian Straub and Richard Newcombe and Steven Lovegrove},
  title     = {{DeepSDF}: Learning Continuous Signed Distance Functions for Shape Representation},
  booktitle = {Proceedings of the IEEE/CVF Conference on Computer Vision and Pattern Recognition (CVPR)},
  year      = {2019},
  pages     = {165--174}
}

@article{Thingi10K,
  title={{Thingi10K}: A Dataset of 10,000 {3D}-Printing Models},
  author={Zhou, Qingnan and Jacobson, Alec},
  journal={arXiv preprint arXiv:1605.04797},
  year={2016}
}

@inproceedings{rosu2023permutosdf,
                        title={{PermutoSDF}: Fast Multi-View Reconstruction with 
                            Implicit Surfaces using Permutohedral Lattices  },
                        author={Radu Alexandru Rosu and Sven Behnke},
                        booktitle={Proceedings of the IEEE/CVF Conference on Computer Vision and Pattern Recognition (CVPR)},
                        year={2023}
                    }

@book{Greenbaum1997,
  title={Iterative Methods for Solving Linear Systems},
  author={Anne Greenbaum},
  publisher={SIAM},
  year={1997}
}

@book{Saad2003,
  title={Iterative Methods for Sparse Linear Systems},
  author={Yousef Saad},
  publisher={SIAM},
  year={2003}
}

@inproceedings{Gupta2018,
  title={Shampoo: Preconditioned stochastic tensor optimization},
  author={Vineet Gupta and Tomer Koren and Yoram Singer},
  booktitle={International Conference on Machine Learning (ICML)},
  year={2018}
}

@inproceedings{Zeiler2014,
  title={Visualizing and Understanding Convolutional Networks},
  author={Matthew D. Zeiler and Rob Fergus},
  booktitle={European Conference on Computer Vision (ECCV)},
  year={2014}
}

@inproceedings{chng2025preconditioners,
  title={Preconditioners for the stochastic training of neural fields},
  author={Chng, Shin-Fang and Saratchandran, Hemanth and Lucey, Simon},
  booktitle={Proceedings of the IEEE/CVF Conference on Computer Vision and Pattern Recognition (CVPR)},
  pages={27222--27232},
  year={2025}
}

@article{park2023camp,
  author    = {Park, Keunhong and Henzler, Philipp and Mildenhall, Ben and Barron, Jonathan T. and Martin-Brualla, Ricardo},
  title     = {{CamP}: Camera Preconditioning for Neural Radiance Fields},
  journal   = {ACM Trans. Graph.},
  volume    = {42},
  number    = {6},
  articleno = {208},
  numpages  = {11},
  year      = {2023},
  doi       = {10.1145/3618321}
}

@inproceedings{Kingma2015Adam,
  title={Adam: A method for stochastic optimization},
  author={Kingma, Diederik P. and Ba, Jimmy},
  booktitle={International Conference on Learning Representations (ICLR)},
  year={2015}
}

@misc{Tieleman2012RMSProp, author={Tieleman, Tijmen and Hinton, Geoffrey}, title={Lecture 6.5 - {RMSProp}: Divide the Gradient by a Running Average of Its Recent Magnitude}, year={2012}}

@inproceedings{chen2022tensorf,
  title={{TensoRF}: Tensorial radiance fields},
  author={Chen, Anpei and Xu, Zexiang and Geiger, Andreas and Yu, Jingyi and Su, Hao},
  booktitle={European Conference on Computer Vision (ECCV)},
  pages={333--350},
  year={2022},
  organization={Springer}
}

@article{sitzmann2020implicit,
  title={Implicit neural representations with periodic activation functions},
  author={Sitzmann, Vincent and Martel, Julien and Bergman, Alexander and Lindell, David and Wetzstein, Gordon},
  journal={Advances in Neural Information Processing Systems (NeurIPS)},
  volume={33},
  pages={7462--7473},
  year={2020}
}

@article{liu2020neural,
  title={Neural sparse voxel fields},
  author={Liu, Lingjie and Gu, Jiatao and Zaw Lin, Kyaw and Chua, Tat-Seng and Theobalt, Christian},
  journal={Advances in Neural Information Processing Systems (NeurIPS)},
  volume={33},
  pages={15651--15663},
  year={2020}
}

@inproceedings{kulhanek2023tetra,
  title={{Tetra-NeRF}: Representing neural radiance fields using tetrahedra},
  author={Kulhanek, Jonas and Sattler, Torsten},
  booktitle={Proceedings of the IEEE/CVF International Conference on Computer Vision (ICCV)},
  pages={18458--18469},
  year={2023}
}

@inproceedings{fridovich2023k,
  title={{K-Planes}: Explicit radiance fields in space, time, and appearance},
  author={Fridovich-Keil, Sara and Meanti, Giacomo and Warburg, Frederik Rahb{\ae}k and Recht, Benjamin and Kanazawa, Angjoo},
  booktitle={Proceedings of the IEEE/CVF Conference on Computer Vision and Pattern Recognition (CVPR)},
  pages={12479--12488},
  year={2023}
}

@article{chen2021nmc,
  author  = {Chen, Zhiqin and Zhang, Hao},
  title   = {Neural Marching Cubes},
  journal = {ACM Trans. Graph.},
  volume  = {40},
  number  = {6},
  pages   = {1--15},
  year    = {2021},
  doi     = {10.1145/3478513.3480518}
}

@article{chen2022ndc,
  author    = {Chen, Zhiqin and Tagliasacchi, Andrea and Funkhouser, Thomas and Zhang, Hao},
  title     = {Neural Dual Contouring},
  journal   = {ACM Trans. Graph.},
  volume    = {41},
  number    = {4},
  articleno = {104},
  pages     = {104:1--104:13},
  year      = {2022},
  doi       = {10.1145/3528223.3530108}
}

@inproceedings{shen2021dmtet,
  title     = {Deep Marching Tetrahedra: a Hybrid Representation for High-Resolution {3D} Shape Synthesis},
  author    = {Shen, Tianchang and Gao, Jun and Yin, Kangxue and Liu, Ming-Yu and Fidler, Sanja},
  booktitle = {Advances in Neural Information Processing Systems (NeurIPS)},
  year      = {2021}
}

@inproceedings{gao2020deftet,
  title     = {Learning Deformable Tetrahedral Meshes for {3D} Reconstruction},
  author    = {Gao, Jun and Chen, Wenzheng and Xiang, Tommy and Tsang, Clement Fuji and Jacobson, Alec and McGuire, Morgan and Fidler, Sanja},
  booktitle = {Advances in Neural Information Processing Systems (NeurIPS)},
  year      = {2020}
}

@inproceedings{mescheder2019occupancy,
  title={Occupancy networks: Learning {3D} reconstruction in function space},
  author={Mescheder, Lars and Oechsle, Michael and Niemeyer, Michael and Nowozin, Sebastian and Geiger, Andreas},
  booktitle={Proceedings of the IEEE/CVF Conference on Computer Vision and Pattern Recognition (CVPR)},
  pages={4460--4470},
  year={2019}
}

@article{kim2018nsml,
  title={{NSML}: Meet the {MLaaS} platform with a real-world case study},
  author={Kim, Hanjoo and Kim, Minkyu and Seo, Dongjoo and Kim, Jinwoong and Park, Heungseok and Park, Soeun and Jo, Hyunwoo and Kim, KyungHyun and Yang, Youngil and Kim, Youngkwan and others},
  journal={arXiv preprint arXiv:1810.09957},
  year={2018}
}

@article{stippel2025marching,
  title     = {Marching Neurons: Accurate Surface Extraction for Neural Implicit Shapes},
  author    = {Stippel, Christian and Mujkanovic, Felix and Leimk{\"u}hler, Thomas and Hermosilla, Pedro},
  journal   = {ACM Trans. Graph.},
  volume    = {44},
  number    = {6},
  articleno = {222},
  pages     = {222:1--222:12},
  year      = {2025},
  doi       = {10.1145/3763328}
}

@inproceedings{sellan2023reach,
  author    = {Sell{\'a}n, Silvia and Batty, Christopher and Stein, Oded},
  title     = {Reach For the Spheres: Tangency-aware Surface Reconstruction of {SDF}s},
  booktitle = {SIGGRAPH Asia 2023 Conference Papers},
  articleno = {73},
  numpages  = {11},
  year      = {2023},
  doi       = {10.1145/3610548.3618196}
}

@article{kohlbrenner2025power,
  author  = {Kohlbrenner, M. and Alexa, M.},
  title   = {Isosurface Extraction for Signed Distance Functions using Power Diagrams},
  journal = {Computer Graphics Forum},
  volume  = {44},
  number  = {2},
  pages   = {e70037},
  year    = {2025},
  doi     = {10.1111/cgf.70037}
}

@inproceedings{ren2024mcgrids,
  author       = {Ren, Daxuan and Shi, Hezi and Zheng, Jianmin and Cai, Jianfei},
  title        = {{McGrids}: {Monte Carlo}-Driven Adaptive Grids for Iso-Surface Extraction},
  booktitle={European Conference on Computer Vision (ECCV)},
  pages        = {127--144},
  year         = {2024},
  organization = {Springer},
  doi          = {10.1007/978-3-031-72998-0_8}
}

@article{shen2023flexicubes,
author = {Shen, Tianchang and Munkberg, Jacob and Hasselgren, Jon and Yin, Kangxue and Wang, Zian 
        and Chen, Wenzheng and Gojcic, Zan and Fidler, Sanja and Sharp, Nicholas and Gao, Jun},
title = {Flexible Isosurface Extraction for Gradient-Based Mesh Optimization},
year = {2023},
issue_date = {August 2023},
publisher = {Association for Computing Machinery},
address = {New York, NY, USA},
volume = {42},
number = {4},
issn = {0730-0301},
doi = {10.1145/3592430},
journal = {ACM Trans. Graph.},
articleno = {37},
numpages = {16}
}

@inproceedings{Fabri2009CGAL,
  author    = {Fabri, Andreas and Pion, Sylvain},
  title     = {{CGAL}: The Computational Geometry Algorithms Library},
  booktitle = {Proceedings of the 17th ACM SIGSPATIAL International Conference on Advances in Geographic Information Systems},
  series    = {GIS '09},
  pages     = {538--539},
  year      = {2009},
  publisher = {Association for Computing Machinery},
  address   = {New York, NY, USA},
  doi       = {10.1145/1653771.1653865}
}
\newpage

\appendix

\makeatletter
\renewcommand*{\theHsection}{A.\arabic{section}}
\renewcommand*{\theHsubsection}{A.\arabic{section}.\arabic{subsection}}
\renewcommand*{\theHfigure}{A.\arabic{figure}}
\renewcommand*{\theHtable}{A.\arabic{table}}
\makeatother

\renewcommand\thesection{A.\arabic{section}}
\renewcommand\thesubsection{A.\arabic{section}.\arabic{subsection}}
\renewcommand\thefigure{A\arabic{figure}}
\renewcommand\thetable{A\arabic{table}}
\onecolumn

{\centering
    \Large
    \textbf{Supplementary Material}\\
    \vspace{1.0em}
}

\setcounter{figure}{0}
\setcounter{table}{0}

{
\setlength{\parindent}{0pt}

\section{Theoretical Proofs}
\label{appx:theory}
In this section, we prove that barycentric interpolation within a tetrahedron is an affine transformation of the input, which we use to show that tetrahedral networks are piecewise affine in~\cref{sec:tetra-networks}. This section reproduces standard properties of barycentric interpolation and the resulting affine mapping in our notation for completeness.
\begin{theorem}[Affine property of barycentric interpolation within a tetrahedron]
\label{prop:affine-barycentric}
Let $\mathbf{g}: \mathbb{R}^3 \to \mathbb{R}^d$ be a function that encodes a point $\mathbf{x}$ into a feature vector $\mathbf{g}(\mathbf{x})$.
For any point $\mathbf{x} \in \mathbb{R}^3$, $\mathbf{g}(\mathbf{x})$ is obtained via barycentric interpolation using the feature vectors that represent the four vertices of a tetrahedron:
\[
\mathbf{g}(\mathbf{v}_0),~\mathbf{g}(\mathbf{v}_1),~\mathbf{g}(\mathbf{v}_2),~\text{and}~\mathbf{g}(\mathbf{v}_3).
\]

Then, $\mathbf{g}(\mathbf{x})$ is an affine transformation of the input $\mathbf{x}$, which can be expressed as:
\begin{equation}\label{eq:affine}
    \mathbf{g}(\mathbf{x}) = \mathbf{A} \mathbf{x} + \mathbf{b}, \quad \mathbf{A} \in \mathbb{R}^{d \times 3}, \quad \mathbf{b} \in \mathbb{R}^{d},
\end{equation}
where $\mathbf{A}$ is a linear transformation matrix, and $\mathbf{b}$ is a translation vector.
\end{theorem}

\begin{proof}
Let $\boldsymbol{w} = (w_0, w_1, w_2, w_3)$ be the barycentric coordinates of a point $\mathbf{x}$ for the tetrahedron defined by its coordinates of vertices $\mathbf{v}_0, \mathbf{v}_1, \mathbf{v}_2, \mathbf{v}_3 \in \mathbb{R}^3$. Each barycentric coefficient $w_i \in \mathbb{R}$ satisfies the constraint: 
$w_0 + w_1 + w_2 + w_3 = 1.$

First, we express the function $\mathbf{g}$ in terms of the barycentric interpolation:
\begin{equation}
    \mathbf{g}(\mathbf{x}) = w_0 \mathbf{g}(\mathbf{v}_0) + w_1 \mathbf{g}(\mathbf{v}_1) + w_2 \mathbf{g}(\mathbf{v}_2) + w_3 \mathbf{g}(\mathbf{v}_3).
\end{equation}
Using the barycentric constraint:
\begin{equation}
    w_0 = 1 - (w_1 + w_2 + w_3),
\end{equation}
we substitute this into the equation:
\begin{align}
    \mathbf{g}(\mathbf{x}) 
    &= (1 - w_1 - w_2 - w_3) \mathbf{g}(\mathbf{v}_0) 
    + w_1 \mathbf{g}(\mathbf{v}_1) 
    + w_2 \mathbf{g}(\mathbf{v}_2) 
    + w_3 \mathbf{g}(\mathbf{v}_3).
\end{align}

Rearranging the equation, we obtain:
\begin{equation}
    \mathbf{g}(\mathbf{x}) = \mathbf{g}(\mathbf{v}_0) +  
    (\mathbf{g}(\mathbf{v}_1) - \mathbf{g}(\mathbf{v}_0))w_1 
    +  (\mathbf{g}(\mathbf{v}_2) - \mathbf{g}(\mathbf{v}_0)) w_2 
    + (\mathbf{g}(\mathbf{v}_3) - \mathbf{g}(\mathbf{v}_0))w_3.
\end{equation}
Defining the matrix $\mathbf{G}$ as:
\begin{equation}
    \mathbf{G} =
    \begin{bmatrix}
        \mathbf{g}(\mathbf{v}_1) - \mathbf{g}(\mathbf{v}_0), ~ \mathbf{g}(\mathbf{v}_2) - \mathbf{g}(\mathbf{v}_0), ~ \mathbf{g}(\mathbf{v}_3) - \mathbf{g}(\mathbf{v}_0)
    \end{bmatrix} \quad \in \mathbb{R}^{d \times 3},
\end{equation}
we can write:
\begin{equation} \label{eq:genc_alpha}
    \mathbf{g}(\mathbf{x}) = \mathbf{g}(\mathbf{v}_0) + \mathbf{G} \begin{bmatrix} w_1 \\ w_2 \\ w_3 \end{bmatrix}.
\end{equation}
Similarly, the point $\mathbf{x}$ itself can be expressed in barycentric form as
\begin{equation}
    \mathbf{x} = w_0 \mathbf{v}_0 + w_1 \mathbf{v}_1 + w_2 \mathbf{v}_2 + w_3 \mathbf{v}_3,
\end{equation}
which, using the same elimination of $w_0$ and grouping terms relative to $\mathbf{v}_0$, yields
\begin{equation}
\label{affine_barycentric_eq}
    \mathbf{x} = \mathbf{v}_0 + \mathbf{C}
    \begin{bmatrix}
        w_1 \\ w_2 \\ w_3
    \end{bmatrix},
\end{equation}
where $\mathbf{C} = [\mathbf{v}_1 - \mathbf{v}_0,\, \mathbf{v}_2 - \mathbf{v}_0,\, \mathbf{v}_3 - \mathbf{v}_0] \in \mathbb{R}^{3 \times 3}$.

If the vectors \( \{ \mathbf{v}_1 - \mathbf{v}_0, \mathbf{v}_2 - \mathbf{v}_0, \mathbf{v}_3 - \mathbf{v}_0 \} \) are linearly independent, \ie, the four vertices $ \mathbf{v}_0, \mathbf{v}_1, \mathbf{v}_2,$ and  $\mathbf{v}_3$ do not lie on the same plane and no three vertices are collinear, then the matrix \( \mathbf{C} \) is invertible.

Thus, the barycentric coordinates of \( \mathbf{x} \) can be computed as:
\begin{equation}\label{barycentric_co}
    \begin{bmatrix} w_1 \\ w_2 \\ w_3 \end{bmatrix} = \mathbf{C}^{-1} (\mathbf{x} - \mathbf{v}_0).
\end{equation}
Substituting \cref{barycentric_co} into \cref{eq:genc_alpha}:
\begin{equation}
    \mathbf{g}(\mathbf{x}) = \mathbf{g}(\mathbf{v}_0) + \mathbf{G} \mathbf{C}^{-1} (\mathbf{x} - \mathbf{v}_0).
\end{equation}
Rewriting this:
\begin{equation}
\label{eq:barycentric_interp_feature}
    \mathbf{g}(\mathbf{x}) = \mathbf{G} \mathbf{C}^{-1} \mathbf{x} + (\mathbf{g}(\mathbf{v}_0) - \mathbf{G} \mathbf{C}^{-1} \mathbf{v}_0).
\end{equation}
Thus, $\mathbf{g}(\mathbf{x})$ is an affine transformation with:
\begin{equation}
\label{eq:Ab_for_affine}
    \mathbf{A} = \mathbf{G} \mathbf{C}^{-1}, \quad \mathbf{b} = \mathbf{g}(\mathbf{v}_0) - \mathbf{G} \mathbf{C}^{-1} \mathbf{v}_0.
\end{equation}
with corresponding $\mathbf{A}$ and $\mathbf{b}$ in \cref{eq:affine}, which completes the proof.
\end{proof}

\begin{lemma}[Affine property of levelwise feature concatenation]
\label{lem:levelwise-affine}
In a region where the feature vector selection is fixed, the levelwise concatenation of the barycentrically interpolated feature vector is also an affine transformation of the input $\mathbf{x}$.
\end{lemma}
\begin{proof}
For each resolution level $\ell \in \{0, \dots, L-1\}$, let
\begin{equation}
\mathbf{g}^{(\ell)}(\mathbf{x})
=
\mathbf{A}^{(\ell)} \mathbf{x} + \mathbf{b}^{(\ell)} \in \mathbb{R}^d
\end{equation}
denote the barycentric feature interpolation within the tetrahedron containing $\mathbf{x}$ at level $\ell$ by ~\cref{eq:barycentric_interp_feature}.
Define the concatenated feature vector $\mathbf{z}(\mathbf{x})$ as:
\begin{equation}
\mathbf{z}(\mathbf{x})
=\bigoplus_{\ell=0}^{L-1}\mathbf{g}^{(\ell)}(\mathbf{x}) = 
\big[\mathbf{g}^{(0)}(\mathbf{x})^\top,~\dots,~\mathbf{g}^{(L-1)}(\mathbf{x})^\top\big]^\top \in \mathbb{R}^{dL},
\end{equation}
where $\oplus$ denotes concatenation over all resolution levels.
In a region where the feature selection is fixed across all levels, we have
\begin{equation}
\mathbf{z}(\mathbf{x})
=
\begin{bmatrix}
\mathbf{A}^{(0)}\\
\vdots\\
\mathbf{A}^{(L-1)}
\end{bmatrix}
\mathbf{x}
+
\begin{bmatrix}
\mathbf{b}^{(0)}\\
\vdots\\
\mathbf{b}^{(L-1)}
\end{bmatrix}
= \mathbf{A}' \mathbf{x} + \mathbf{b}'
\end{equation}
for suitable $\mathbf{A}'$ and $\mathbf{b}'$.
Thus, $\mathbf{z}(\mathbf{x})$ is an affine transformation of the input $\mathbf{x}$, which completes the proof.
\end{proof}

\section{Input Preconditioning}
\label{supp:global_preconditioner}
Our global input preconditioner mitigates the geometric anisotropy introduced by the tetrahedral encoder by
mapping $\mathbf{x}$ to $\mathbf{A}^*\mathbf{x}$.
Since the learnable features are interpolated using barycentric coefficients that are derived from an affine function of the input $\mathbf{x}$ inside each tetrahedron from ~\cref{barycentric_co}, the mapping from input to barycentric coordinates should not favor specific directions.
Although the encoder is multi-resolution, each resolution level is a scaled and translated copy of the same tetrahedral tiling, and the resulting features from each level are only concatenated (\textit{ref.} \cref{lem:levelwise-affine}).
Therefore, it suffices to analyze the encoder-induced metric on a single canonical unit cube and its six tetrahedra; the same preconditioner applies globally to all cells and levels.

\paragraph{Encoder-induced local metric}
Consider a tetrahedron $\mathcal{T}$ from the subdivision of the cube
shown in Figure~\ref{fig:tetsubdiv_mtd}; for convenience, we treat it as a unit cube with ordered vertices $(\mathbf{v}_0,\mathbf{v}_1,\mathbf{v}_2,\mathbf{v}_3)$.
Define
\begin{equation}
\mathbf{C}_\mathcal{T}
=
\big[\,\mathbf{v}_1-\mathbf{v}_0\;\; \mathbf{v}_2-\mathbf{v}_0\;\; \mathbf{v}_3-\mathbf{v}_0\,\big]
\in\mathbb{R}^{3\times3},
\quad
\mathbf{w}=[w_1,w_2,w_3]^\top.
\end{equation}
Here, $\mathbf{v}_0$ is the grid anchor of Fig.~\ref{fig:tetsubdiv_mtd} (the lexicographically smallest cube corner), consistent with $\mathbf{a}^{(\ell)}(\mathbf{x})$ in Sec.~4.4.
Inside $\mathcal{T}$, ~\cref{barycentric_co} yields $\mathbf{x}=\mathbf{v}_0+\mathbf{C}_\mathcal{T}\mathbf{w},$ and differentiation gives $d\mathbf{w}=\mathbf{C}_\mathcal{T}^{-1}d\mathbf{x}$.
Measuring the squared norm of $d\mathbf{w}$ gives:
\begin{equation}
\label{eq:local_metric_MT}
\|d\mathbf{w}\|^2
= \big\|\mathbf{C}_\mathcal{T}^{-1} d\mathbf{x}\big\|^2
= d\mathbf{x}^\top \mathbf{C}_\mathcal{T}^{-\top}\mathbf{C}_\mathcal{T}^{-1} d\mathbf{x},
\quad
\mathbf{M}_\mathcal{T} := \mathbf{C}_\mathcal{T}^{-\top}\mathbf{C}_\mathcal{T}^{-1}.
\end{equation}
This encoder-induced local metric quantifies how sensitively barycentric coordinates respond to infinitesimal displacements in $\mathbf{x}$.
Our tetrahedral encoder is obtained by regularly tiling space with scaled and translated copies of a canonical unit cube subdivided into six congruent tetrahedra (see~\cref{fig:tetsubdiv_mtd}).
Let $\mathcal{T} \in \{S_0,\dots,S_5\}$ denote these tetrahedra and let $\mathbf{M}_\mathcal{T}$ be their local metrics
from~\cref{eq:local_metric_MT}.
Their simple average represents the average encoder-induced anisotropy within one canonical cube cell:
\begin{equation}
\label{eq:M_cell_avg}
\mathbf{M}
= \frac{1}{6}\sum_{\mathcal{T}\in\{S_0,\dots,S_5\}} \mathbf{M}_\mathcal{T}
=
\frac{1}{3}
\begin{bmatrix}
5 & -2 & -2\\
-2 & 5 & -2\\
-2 & -2 & 5
\end{bmatrix}.
\end{equation}
The cube-averaged metric $\mathbf{M}$ is invariant under permutations of the
coordinate axes. Equivalently, averaging $\mathbf{M}$ over all axis permutations
leaves it unchanged.
Let $\mathbf{P}^{(0)},\dots,\mathbf{P}^{(5)}$ be the six $3\times 3$
permutation matrices
\[
\mathbf{P}^{(0)} =
\begin{bmatrix}
1&0&0\\[2pt]0&1&0\\[2pt]0&0&1
\end{bmatrix},\quad
\mathbf{P}^{(1)} =
\begin{bmatrix}
0&1&0\\[2pt]1&0&0\\[2pt]0&0&1
\end{bmatrix},\quad
\mathbf{P}^{(2)} =
\begin{bmatrix}
1&0&0\\[2pt]0&0&1\\[2pt]0&1&0
\end{bmatrix},
\]
\[
\mathbf{P}^{(3)} =
\begin{bmatrix}
0&0&1\\[2pt]1&0&0\\[2pt]0&1&0
\end{bmatrix},\quad
\mathbf{P}^{(4)} =
\begin{bmatrix}
0&1&0\\[2pt]0&0&1\\[2pt]1&0&0
\end{bmatrix},\quad
\mathbf{P}^{(5)} =
\begin{bmatrix}
0&0&1\\[2pt]0&1&0\\[2pt]1&0&0
\end{bmatrix}.
\]
We then define:
\begin{equation}
\label{eq:cell_permutation}
\mathbf{M}_{\mathrm{sym}}
= \frac{1}{6}\sum_{k=0}^5 \mathbf{P}^{(k)} \mathbf{M} \mathbf{P}^{(k)\top}
= \frac{1}{3}
\begin{bmatrix}
5 & -2 & -2\\
-2 & 5 & -2\\
-2 & -2 & 5
\end{bmatrix}.
\end{equation}
Thus, $\mathbf{M}_{\mathrm{sym}}=\mathbf{M}$; we retain the notation
$\mathbf{M}_{\mathrm{sym}}$ to emphasize its invariance under axis relabeling.
It has eigenvalues $\big(\tfrac{1}{3}, \tfrac{7}{3}, \tfrac{7}{3}\big)$, with $\mathbf{u}_\parallel = \tfrac{1}{\sqrt{3}}(1,1,1)^\top$ as eigenvector for $\tfrac{1}{3}$, and the eigenspace for $\tfrac{7}{3}$ given by
$\{\mathbf{u}\in\mathbb{R}^3 \mid \mathbf{u}^\top \mathbf{u}_\parallel = 0\}$.
Hence the spectral condition number,
defined as the ratio of the largest to the smallest eigenvalue,
is
\begin{equation}
\kappa(\mathbf{M}_{\mathrm{sym}})
= \frac{\lambda_{\max}}{\lambda_{\min}}
= 7.
\end{equation}
Thus, with respect to this average metric, the encoder-induced geometry attenuates motion
along the body diagonal and amplifies any transverse direction by a factor $\sqrt{7}$
in magnitude, on average over canonical cube cells.

\paragraph{Global linear preconditioner}
We start by expressing the global preconditioner as an affine map in homogeneous coordinates:
\begin{equation}
\tilde{\mathbf{y}} = \mathbf{T}\tilde{\mathbf{x}},
\qquad
\mathbf{T} =
\begin{bmatrix}
\mathbf{A}^* & \mathbf{t}\\
\mathbf{0}^\top & 1
\end{bmatrix},
\qquad
\tilde{\mathbf{x}} =
\begin{bmatrix}
\mathbf{x} \\[2pt] 1
\end{bmatrix},
\quad
\tilde{\mathbf{y}} =
\begin{bmatrix}
\mathbf{y} \\[2pt] 1
\end{bmatrix}.
\end{equation}
where $\mathbf{A}^* \in \mathbb{R}^{3\times3}$ is invertible and
$\mathbf{t}\in\mathbb{R}^3$ is an arbitrary translation vector.
Then, inside a tetrahedron $\mathcal{T}$, we have
$\mathbf{w}(\mathbf{y}) = \mathbf{C}_\mathcal{T}^{-1}(\mathbf{y}-\mathbf{v}_0)$
with $\mathbf{y} = \mathbf{A}^*\mathbf{x} + \mathbf{t}$, and hence, as a function of $\mathbf{x}$,
\begin{equation}
\mathbf{w}(\mathbf{x})
= \mathbf{C}_\mathcal{T}^{-1}(\mathbf{A}^*\mathbf{x} + \mathbf{t} - \mathbf{v}_0).
\end{equation}
Differentiating with respect to $\mathbf{x}$ gives:
\begin{equation}
\label{eq:Meff_local}
d\mathbf{w}
= \mathbf{C}_\mathcal{T}^{-1}\mathbf{A}^*\,d\mathbf{x},
\quad
\|d\mathbf{w}\|^2
= d\mathbf{x}^\top \mathbf{A}^{*\top}\mathbf{M}_\mathcal{T}\mathbf{A}^*\,d\mathbf{x},
\end{equation}
where $\mathbf{M}_\mathcal{T}$ is the original local metric from ~\cref{eq:local_metric_MT}.
Here the translation $\mathbf{t}$ cancels in the differential and does not affect the metric.
Thus the linear map $\mathbf{A}^*$ genuinely changes the encoder-induced geometry in the input space. We now average these preconditioned local metrics over the canonical cube and over all axis permutations.
Using the permutation matrices $\mathbf{P}^{(k)}$ and the definition of
$\mathbf{M}_{\mathrm{sym}}$ from ~\cref{eq:M_cell_avg} and~\cref{eq:cell_permutation}, we obtain 
\begin{equation}
\label{eq:Msym_eff_double}
\frac{1}{36}
\sum_{k=0}^5\
\sum_{\mathcal{T}\in\{S_0,\dots,S_5\}}
\mathbf{A}^{*\top}\mathbf{P}^{(k)}\mathbf{M}_\mathcal{T}\mathbf{P}^{(k)\top}\mathbf{A}^*
= \mathbf{A}^{*\top}\mathbf{M}_{\mathrm{sym}}\mathbf{A}^*.
\end{equation}
Thus, the effective average metric seen from the input space is
$\mathbf{A}^{*\top}\mathbf{M}_{\mathrm{sym}}\mathbf{A}^*$.
We choose $\mathbf{A}^*$ to \textit{isotropize} this average and enforce volume preservation to avoid global rescaling:
\begin{equation}
\label{eq:isotropic_condition}
\mathbf{A}^{*\top}\mathbf{M}_{\mathrm{sym}} \mathbf{A}^* = c\,\mathbf{I},
\quad
\det(\mathbf{A}^*) = 1.
\end{equation}
Under these conditions, the preconditioned average encoder-induced metric is isotropic and has unit  condition number $\kappa(\mathbf{A}^{*\top}\mathbf{M}_{\mathrm{sym}} \mathbf{A}^*) = 1$:
\begin{equation}
d\mathbf{x}^\top \mathbf{A}^{*\top}\mathbf{M}_{\mathrm{sym}} \mathbf{A}^* d\mathbf{x}
= c\,\|d\mathbf{x}\|^2.
\end{equation}
Thus, with respect to the symmetrized cube-average metric, all directions in the preconditioned input space are treated equally; on average over canonical cells, the encoder no longer imposes a preferred direction on infinitesimal displacements.

\cref{eq:isotropic_condition} determines $\mathbf{A}^*$ only up to right multiplication by a rotation.
Choosing the symmetric positive-definite solution gives
$\mathbf{A}^*=\sqrt{c}\,\mathbf{M}_{\mathrm{sym}}^{-1/2}$, where the volume-preserving constraint $\det(\mathbf{A}^*)=1$ yields
$c=\det(\mathbf{M}_{\mathrm{sym}})^{1/3}=7^{2/3}/3$.
Thus, the closed-form preconditioner is:
\begin{equation}
\label{eq:A_closed_form}
\mathbf{A}^*
=
\frac{7^{1/3}}{3}
\begin{bmatrix}
1+\tfrac{2}{\sqrt{7}} & 1-\tfrac{1}{\sqrt{7}} & 1-\tfrac{1}{\sqrt{7}}\\
1-\tfrac{1}{\sqrt{7}} & 1+\tfrac{2}{\sqrt{7}} & 1-\tfrac{1}{\sqrt{7}}\\
1-\tfrac{1}{\sqrt{7}} & 1-\tfrac{1}{\sqrt{7}} & 1+\tfrac{2}{\sqrt{7}}
\end{bmatrix}.
\end{equation}
which satisfies $\det(\mathbf{A}^*)=1$ and
$\mathbf{A}^{*\top}\mathbf{M}_{\mathrm{sym}}\mathbf{A}^* = c\,\mathbf{I}$
with $c = 7^{2/3}/3$.
Numerically, we use
\begin{equation}
\mathbf{A}^*\approx
\begin{bmatrix}
1.11965708 & 0.39663705 & 0.39663705\\
0.39663705 & 1.11965708 & 0.39663705\\
0.39663705 & 0.39663705 & 1.11965708
\end{bmatrix}.
\end{equation}

To quantify the effect of this preconditioner, \cref{tab:cond_numbers_precond}
reports spectral condition numbers $\kappa(\cdot)$ of the encoder-induced metrics on
the canonical cube before and after applying $\mathbf{A}^*$.
\begin{table}[H]
\centering
\small
\captionof{table}{
Condition numbers of encoder-induced metrics on the canonical unit cube
with the tetrahedral subdivision shown in~\cref{fig:tetsubdiv_mtd}, before and after
applying the global preconditioner $\mathbf{A}^*$.
}
\label{tab:cond_numbers_precond}
\begin{tabular}{lcc}
\toprule
Case & $\kappa(\mathbf{M})$ & $\kappa(\mathbf{A}^{*\top} \mathbf{M}\mathbf{A}^*)$ \\
\midrule
$\mathbf{M}_{\mathrm{sym}}$ & $7.00$ & $\mathbf{1.00}$ \\
$\mathbf{M}_{S_0}$ & $16.39$ & $\mathbf{5.05}$ \\
$\mathbf{M}_{S_1}$ & $16.39$ & $\mathbf{5.05}$ \\
$\mathbf{M}_{S_2}$ & $16.39$ & $\mathbf{5.05}$ \\
$\mathbf{M}_{S_3}$ & $16.39$ & $\mathbf{5.05}$ \\
$\mathbf{M}_{S_4}$ & $16.39$ & $\mathbf{5.05}$ \\
$\mathbf{M}_{S_5}$ & $16.39$ & $\mathbf{5.05}$ \\
\bottomrule
\end{tabular}
\end{table}

\section{Polyhedral Cells}
\label{supp:convex-cells}

Recall from~\cref{mtd:tet-enc} that the multi-resolution tetrahedral
encoder partitions the bounded domain $\mathcal{S} \subset \mathbb{R}^3$
into $\mathbf{x}$-containing tetrahedra
$\mathcal{T}^{(\ell)}(\mathbf{x})$ at each resolution level
$\ell \in \{0, 1, \dots, L-1\}$.
For $\mathbf{x} \in \mathcal{S}$, we define the encoder-induced polyhedral cell of $\mathbf{x}$ as:
\begin{equation}
    \mathcal{C}_{\mathbf{x}}
    :=
    \bigcap_{\ell=0}^{L-1} \mathcal{T}^{(\ell)}(\mathbf{x}).
\end{equation}

\begin{lemma}[Convexity of polyhedral cells]
For any $\mathbf{x} \in \mathcal{S}$, the polyhedral cell $\mathcal{C}_{\mathbf{x}}$
is a convex polyhedron.
\end{lemma}

\begin{proof}
Each tetrahedron $\mathcal{T}^{(\ell)}(\mathbf{x})$ can be written as
the intersection of four half-spaces of the form
\begin{equation}
\mathcal{T}
=
\left\{
  \mathbf{y} \in \mathbb{R}^3
  \,\middle|\,
  \mathbf{a}_j^\top \mathbf{y} \le b_j,
  \; j \in \{0,1,2,3\},\;
  b_j \in \mathbb{R}
\right\},
\end{equation}
and half-spaces are convex by definition.
Hence each $\mathcal{T}^{(\ell)}(\mathbf{x})$ is convex.
Because the intersection of convex sets remains
convex~\cite{boyd2004convex}, the cell $\mathcal{C}_{\mathbf{x}}$ is
also convex.
\end{proof}


\begin{remark}[restated]
Notice that the encoder $\tau$ is piecewise affine and continuous over
$\mathcal{S}$.
When $L = 1$, $\tau$ is affine with respect to $\mathbf{x}$ within each
tetrahedron, while for $L > 1$, $\tau$ is affine with respect to
$\mathbf{x}$ within each cell $\mathcal{C}_{\mathbf{x}}$.
In the special case where the geometric progression ratio
\begin{equation}
    \gamma := \exp\!\left( \frac{\ln N_{\max} - \ln N_{\min}}{L-1} \right)
\end{equation}
satisfies $\gamma \in \mathbb{N}$ and $\gamma \ge 2$, every cell $\mathcal{C}_{\mathbf{x}}$ is a tetrahedron at the finest level.
\end{remark}

\section{Implementation Details of Initial Skeleton Extraction}
\label{sec:polycomplex-construction}
In the initial skeleton extraction step, our goal is to extract all vertices and edges of $\mathcal{C}$ from the multi-resolution tetrahedral grid.
This section provides the detailed algorithmic construction described in~\cref{sec:grid-skeleton}.
\subsection{Grid Representation as a Union of Parallel-plane Sets}
We exploit the inherent regularity of the multi-resolution tetrahedral grid to represent $\mathcal{C}$ using a union of parallel-plane sets. As the slopes of tetrahedral subdivision are constant across levels, their normal set is identical while their offsets differ across levels.
Thus, multiple planes can be compactly represented as sharing a single normal but varying in their offsets.
We denote these unique normals as:
\begin{align}
    \mathcal{N}=\{\vn_0, \vn_1, \dots, \vn_i, \dots \}.
\end{align}
For each $\vn_i$, the associated offsets across all resolution levels are defined as a vector as follows:
\begin{equation}
\vd^{(i)} = \big[ d^{(i)}_0, d^{(i)}_1, \dots \big]^\top.
\end{equation} 
This representation makes it easier to extract the vertices and edges of $\mathcal{C}$.  
Moreover, the formulation scales efficiently to large numbers of vertices and edges.

\subsection{Vertex Extraction}
\label{sec:input}

A vertex of $\mathcal{C}$ is the intersection of at least three planes whose normals are linearly independent. 
We define the set of such normal triplets by:
\begin{equation}
\mathcal{T}^{(ijk)}
:=
\bigl\{
\{\vn_i,\vn_j,\vn_k\}\in [\mathcal{N}]^3
\;\big|\;
\det[\vn_i\ \vn_j\ \vn_k]\neq 0
\bigr\},
\label{eq:triplet-set}
\end{equation}
where $[\mathcal{N}]^k$ denotes the set of all $k$-element subsets of $\mathcal{N}$.
For each triplet, we define the normal matrix as:
\begin{align}
\mN^{(ijk)}
:=
\begin{bmatrix}
\vn_i & \vn_j & \vn_k
\end{bmatrix}
\in \mathbb{R}^{3\times 3},
\label{eq:normal-matrix}
\end{align}
and the corresponding offset matrix is given by:
\begin{equation}
\mD^{(ijk)} \;=\;
\begin{bmatrix}
\bigl(\vd^{(i)} \otimes \mathbf1_{\,|\vd^{(j)}|} \otimes \mathbf1_{\,|\vd^{(k)}|}\bigr)^\top \\
\bigl(\mathbf 1_{\,|\vd^{(i)}|} \otimes \vd^{(j)} \otimes \mathbf 1_{\,|\vd^{(k)}|}\bigr)^\top \\
\bigl(\mathbf 1_{\,|\vd^{(i)}|} \otimes \mathbf 1_{\,|\vd^{(j)}|} \otimes \vd^{(k)}\bigr)^\top
\end{bmatrix}
\in \mathbb R^{3\times (|\vd^{(i)}||\vd^{(j)}||\vd^{(k)}|)},
\label{eq:D-kron}
\end{equation}
where $\mathbf{1}_{|\vd|}$ denotes the column vector of ones of length $|\vd|$ and $\otimes$ denotes the Kronecker product, which can be viewed as the tensor product collapsed into a vector in our case. 

Thus, the vertices for the chosen normal triplet $\mathcal{T}^{(ijk)}$ are the solution of linear equation:
\begin{align}
    \mN^{(ijk)\top} \mX^{(ijk)} = \mD^{(ijk)}
\end{align}

which can be obtained by:
\begin{equation}
\mathbf{X}^{(ijk)} = \mN^{(ijk)-\top} \mathbf{D}^{(ijk)}.
\label{eq:vertex-batch}
\end{equation}
Notice that each column of $\mathbf{X}^{(ijk)}$ corresponds to a vertex. By collecting these columns over all triplets of $\mathcal T$ yields the candidate vertex set $\mathcal X_{\mathrm{all}}$.
Since planes are unbounded, some candidate vertices may fall outside the bounded domain $\mathcal S$. We therefore define the final extracted vertex set as:
\begin{equation}
\mathcal{V} := \mathcal{X}_{\mathrm{all}} \cap \mathcal{S}.
\label{eq:final-vertices}
\end{equation}
Note that the computation can be efficiently parallelized over triplets via tensor broadcasting.

\subsection{Edge Extraction}
\label{sec:edge-recon}
An edge of $\mathcal{C}$ is the intersection of at least two non-parallel planes. Similarly, we define the set of normal pairs describing the two planes as follows:
\begin{equation}
\mathcal{P} := \{\{\vn_i,\vn_j\}\in [\mathcal{N}]^2 \mid \vn_i \times \vn_j \neq \mathbf{0}\}.
\end{equation}
For each $\{\vn_i,\vn_j\} \in \mathcal{P}$, we define
$
\mathbf{u}_{ij} := \frac{\vn_i \times \vn_j}{||\vn_i \times \vn_j||}
$
as the unit direction vector of the line of intersection between the two planes. Note that the offsets of the planes are not required for this.

Given the above normals $\vn_i$ and $\vn_j$ together with their offsets $d_n^{(i)}$ and $d_m^{(j)}$ and the previously extracted vertex set $\mathcal{V}$, the endpoints of the edges are required to satisfy:
\begin{equation}
\mathcal{V}_{nm}^{(ij)}
:=
\{\vv \in \mathcal{V} \mid 
\vn_i^\top \vv = d_n^{(i)},\
\vn_j^\top \vv = d_m^{(j)}\}.
\end{equation}
However, generating edges directly from $\mathcal{V}_{mn}^{(ij)}$ may result in overlaps. 
To avoid this, we order them by the projection of each vertex onto $\mathbf{u}_{ij}$:
\begin{equation}
t^{(ij)}_{mn}(\vv) := \mathbf{u}_{ij}^\top \vv,\quad \vv\in\mathcal{V}_{mn}^{(ij)} .
\end{equation}
This scalar projection represents the distance from the origin. The sequence obtained by sorting $\mathcal{V}_{mn}^{(ij)}$ in ascending $t_{ij}$ order is as follows:
\begin{align}
    \widetilde{\mathcal{V}}_{mn}^{(ij)} = \big[ \vv_0, \dots, \vv_\ell, \vv_{\ell+1}, \dots \big].
\end{align}
The non-overlapping local edge set is then obtained by connecting consecutive vertices in $\widetilde{\mathcal{V}}_{mn}^{(i,j)}$:
\begin{equation}
\mathcal{E}_{mn}^{(ij)}
:=
\bigl\{\{\vv_{\ell},\,\vv_{\ell+1}\} \mid \vv_\ell, \vv_{\ell+1} \in \widetilde{\mathcal{V}}_{mn}^{(i,j)}, 0 \le \ell < |\widetilde{\mathcal{V}}_{mn}^{(i,j)}|-1 \bigr\}.
\end{equation}
The whole edge set $\mathcal{E}$ is obtained by taking the union of $\mathcal{E}_{mn}^{(ij)}$ over all normal pair and offset pair combinations.
Since the vertex set $\mathcal{V}$ is already defined within the bounded domain, no additional post-processing is required to be within $\mathcal{S}$.
Once again, this computation can be parallelized efficiently using tensor broadcasting.
\clearpage
\section{Additional Results}
\label{appendix:add_results}
\subsection{Runtime}
\label{appendix:runtime}
We report runtime on the Stanford dataset across resolution settings.
\cref{tab:runtime_memory} reports the runtime for initial skeleton extraction and the end-to-end total runtime, which includes skeleton extraction. It also lists the initial-skeleton sizes $|\mathcal{V}|$ and $|\mathcal{E}|$, and the peak GPU memory.
The results show that our method extracts the large vertex and edge sets of the initial skeleton---encoder-induced polyhedral complex---efficiently while keeping the total runtime reasonable.
When the encoder’s levels and resolutions are fixed, the initial skeleton can be cached and reused. For reference, MC and DC implementations of CGAL~\cite{Fabri2009CGAL} at \(512^3\) resolution take 0.94s and 5.87s on average over the Stanford dataset, respectively.
\begin{table}[htbp]
\caption{Runtime decomposed into initial skeleton extraction and total runtime, with peak GPU memory, for each resolution setting $R$ on the Stanford dataset. When the encoder settings are fixed, the initial skeleton can be cached and reused.}
\label{tab:runtime_memory}
\centering
\begin{tabular}{lccccc}
\toprule
\multirow{2}{*}{$R$} &
\multicolumn{3}{c}{Initial skeleton extraction} &
\multirow{2}{*}{Total (s)} &
\multirow{2}{*}{\shortstack{Memory\\(GB)}} \\
\cmidrule(lr){2-4}
& $|\mathcal{V}|$ & $|\mathcal{E}|$ & Time (s) && \\
\midrule
Small  & $4.4\times 10^{5}$ & $1.6\times 10^{6}$ & 0.24 & 1.83 & 0.99 \\
Medium & $4.4\times 10^{6}$ & $1.6\times 10^{7}$ & 0.33 & 1.42 & 1.53 \\
Large  & $3.5\times 10^{7}$ & $1.3\times 10^{8}$ & 1.99 & 4.56 & 8.01 \\
\bottomrule
\end{tabular}%
\end{table}


\subsection{Runtime and Capacity Comparison with Analytic Marching}
\label{appendix:am-runtime}
We further compare Analytic Marching (AM) applied to ReLU MLPs with different numbers of hidden layers and widths.
The AM setting used in the main paper is AM-H10/W90, which denotes a ReLU MLP with 10 hidden layers of width 90.
As shown in \cref{tab:runtime_am}, the AM baseline in the main paper has runtime comparable to Ours-Large but much higher CD, and larger ReLU MLP variants do not close the gap.
Ours-Medium is both faster and more accurate than AM-H20/W200 despite having fewer parameters, suggesting that the gain comes from the tetrahedral encoder improving the learned SDF while preserving analytic isosurfacing.
Note that ours runtime includes initial skeleton extraction.
\begin{table}[htbp]
\caption{ Runtime and capacity comparison with Analytic Marching (AM) on the Stanford dataset. CD is reported in units of $10^{-6}$. Note that ours runtime includes initial skeleton extraction. } \label{tab:runtime_am}
\centering \begin{tabular}{lccc}
\toprule
Method & Params & CD$\downarrow$ & Time (s)$\downarrow$ \\
\midrule
AM-H10/W90 & 74K & 5475 & 4.21 \\
AM-H15/W150 & 318K & 4389 & 3.82 \\
AM-H20/W200 & 765K & 5062 & 5.10 \\
Ours-Medium & 588K & 2462 & 1.42 \\
Ours-Large & 4.5M & 1718 & 4.56 \\
\bottomrule
\end{tabular}
\end{table}
\subsection{Comparison with Adaptive Sampling-based Extraction}
\label{appendix:additional-runtime-extractor}

We also evaluate adaptive sampling-based extractors on our trained network (Large), including CGAL~\cite{Fabri2009CGAL} implementations of octree Marching Cubes (OctMC) and octree Dual Contouring (OctDC), and McGrids~\cite{ren2024mcgrids} with its official default settings. We use the Stanford dataset for this comparison. As shown in \cref{tab:octree_comparison}, OctMC, OctDC, and McGrids require substantially more vertices and runtime than Ours while producing worse self-consistency with the same trained network.
At depth 12, octree-based extraction became impractical on the Stanford dataset; Bunny alone produced 71M leaves, indicating that octree adaptivity does not prevent near-surface explosion on these curved shapes.
\begin{table}[htbp]
\caption{ Adaptive sampling-based extractor comparison on the same trained network. SSDF and VSDF are reported in units of $10^{-6}$.}
\label{tab:octree_comparison}
\centering
\begin{tabular}{lccccc}
\toprule
Method & SSDF$\downarrow$ & VSDF$\downarrow$ & AD($^\circ$)$\downarrow$ & $|\mathcal{V}|\downarrow$ & Time (s)$\downarrow$ \\
\midrule
OctMC & 8.85 & 3.28 & 1.92 & 6.76M & 124.2 \\
OctDC & 6.66 & 5.18 & 1.76 & 13.49M & 139.3 \\
McGrids & 332.2 & 16.5 & 3.13 & 3.73M & 69.6 \\
Ours & 0.076 & 0.078 & 0.00 & 0.28M & 4.56 \\
\bottomrule
\end{tabular}
\end{table}

\newpage
\subsection{Qualitative Comparison across Resolution Settings}
\cref{fig:ablation_res} shows the qualitative effect of the three
resolution settings (\emph{Small}, \emph{Medium}, \emph{Large}) on the
Stanford Dragon, complementing the quantitative results in
\cref{tab:different_resolution}.
As the resolution decreases, TropicalNeRF loses fine details much more rapidly than our method. Our method still produces visually plausible meshes even under the \emph{Small} setting.
\begin{figure}[htbp]
  \centering
  \includegraphics[width=0.8\linewidth]{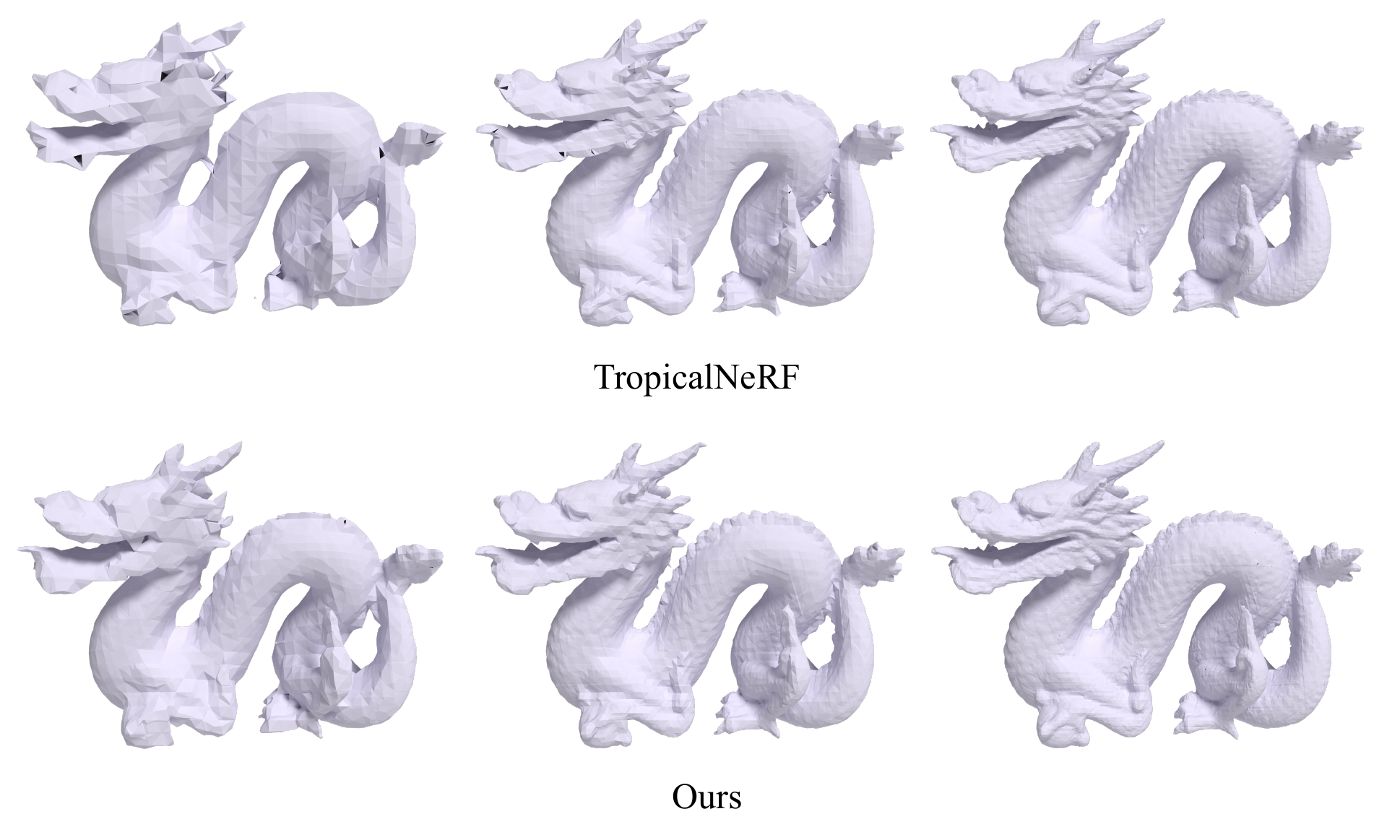}
  \caption{
    Qualitative results on the Stanford Dragon across resolution settings:
    \emph{Small} (left), \emph{Medium} (middle), and \emph{Large} (right).
  }
  \label{fig:ablation_res}
\end{figure}

\section{Further Discussions}
\label{sec:appendix_further_discussions}

\subsection{Design Choice for Our Tetrahedral Grid}
\label{sec:appendix_6tet}
Our tetrahedral grid is designed to satisfy two practical requirements for both analytic extraction and training stability:
\begin{itemize}
  \item \textbf{Efficient Analytic Traversal:} given a query point $\mathbf{x}$, we should be able to identify its containing tetrahedron and retrieve a compact neighborhood under a single canonical local pattern repeated throughout space, enabling fixed-size lookup tables with minimal case distinctions.
  \item \textbf{Uniform Vertex Valence:} since features are stored at grid vertices and accessed by interpolation, vertices should have uniform valence (\ie, the same number of incident tetrahedra) to avoid systematically uneven update frequency and gradient variance across vertices.
\end{itemize}

Our congruent 6-tet cube subdivision satisfies these requirements by design.
Moreover, the remaining directional bias induced by the grid anisotropy is addressed by our global preconditioner $\mathbf{A}^*$ (\cref{sec:precondition}), which improves conditioning while preserving the same indexing and neighbor structure.

In contrast, a 5-tet decomposition is non-congruent and often requires multiple local patterns across the grid, increasing case distinctions for point-to-tetrahedron indexing and complicating compact neighbor lookup across cell boundaries.
A 12-tet decomposition with a Steiner center vertex introduces multiple vertex types with different valence, which can lead to uneven feature-hit counts and larger traversal neighborhoods.

\subsection{Applying $\mathbf{A}^*$ to HashGrid}
\label{sec:appendix_Astar_hashgrid}
As noted in~\cref{sec:gt_accuracy,tab:ablation_a}, applying our encoder-derived $\mathbf{A}^*$ improves our model's accuracy but does not improve HashGrid's accuracy, indicating that $\mathbf{A}^*$ is not a generic preconditioner.
Here we briefly explain why a \emph{global linear} preconditioner is not expected to provide directional gains for HashGrid's within-cell trilinear interpolation.

Following~\cref{supp:global_preconditioner}, consider a canonical unit cube with the standard trilinear interpolation weights for its $8$ corner vertices.
Let $a_0(t)=1-t$ and $a_1(t)=t$ be two linear functions on $t\in[0,1]$; then the weight for corner $(i,j,k)\in\{0,1\}^3$ is
\begin{equation}
w_{ijk}(x,y,z)=a_i(x)\,a_j(y)\,a_k(z), \qquad (x,y,z)\in[0,1]^3.
\end{equation}
Let $\mathbf{x}=(x,y,z)^\top$ and let $\mathbf{w}(\mathbf{x})\in\mathbb{R}^8$ stack the eight weights $\{w_{ijk}\}$.
Differentiation gives $d\mathbf{w}=\mathbf{J}\,d\mathbf{x}$.
We define the cube-average metric $\mathbf{M}_{\mathrm{tri}}\in\mathbb{R}^{3\times 3}$ as the canonical-cube average of $\mathbf{J}^\top\mathbf{J}$.
By cube symmetries, $\mathbf{M}_{\mathrm{tri}}=\alpha\mathbf{I}$ for some $\alpha>0$.
Under the same volume-preserving isotropization condition used in~\cref{supp:global_preconditioner},
we enforce $\mathbf{A}^\top\mathbf{M}_{\mathrm{tri}}\mathbf{A}=c\mathbf{I}$ together with $\det(\mathbf{A})=1$.
Since $\mathbf{M}_{\mathrm{tri}}=\alpha\mathbf{I}$, this implies $\mathbf{A}^\top\mathbf{A}=\mathbf{I}$ and thus we can take $\mathbf{A}=\mathbf{I}$.
Therefore, a global linear preconditioner is not expected to provide directional gains for HashGrid.

\subsection{Limitations}
\label{sec:appendix_triangle_quality}
Triangle regularity\footnote{Here, triangle regularity refers to avoiding skinny triangles, \eg, larger minimum angles.} can be important for downstream geometry processing tasks such as deformation-based surface editing, but it often trades off with the exact agreement with the learned zero-level set.
Sampling-based isosurface extraction (\eg, MC, MT, and DC) approximates the zero-level set under a discrete grid, incurring discretization error but yielding meshes whose triangle regularity is largely determined by the sampling scheme.
In contrast, analytic isosurface extraction (\eg, AM, TropicalNeRF, and ours) prioritizes exact agreement with the learned zero-level set.
In~\cref{tab:tri_quality}, we report the fraction of faces with $\theta_{\min}<5^\circ$.
When required, triangle regularity can be improved via standard post-processing, at the cost of deviating from the learned zero-level set.

\begin{table}[ht]
\setlength{\abovecaptionskip}{0pt}%
\setlength{\belowcaptionskip}{0pt}%
\captionof{table}{Fraction of faces with minimum angle below $5^\circ$ for field-faithful analytic isosurface extractors on the Stanford dataset. $\theta_{\min}$ denotes the per-face minimum angle. Values are averaged over the Stanford shapes; TropicalNeRF and ours use the \emph{Large} resolution setting (\cref{sec:exp}).}
\label{tab:tri_quality}
\centering
\scriptsize
\setlength{\tabcolsep}{4pt}
\renewcommand{\arraystretch}{1.05}
\begin{tabular}{lc}
\toprule
Method & $\%\{\theta_{\min}<5^\circ\}$ \\
\midrule
AM & 10.05 \\
TropicalNeRF & 10.94 \\
Ours & 11.83 \\
\bottomrule
\end{tabular}
\end{table}
\section{Edge Subdivision Intuition}
\label{supp:edge-subdiv}
In this section, we briefly recap the edge subdivision scheme from prior work~\cite{berzins2023polyhedral} (Sec.~\ref{sec:edge-subdivision}) and provide intuition for readers who are less familiar with this topic; our grid-aware extensions are described in~\cref{sec:region-indicators} and~\cref{sec:perturbation}.

\noindent\textbf{Notation}~~
\begin{itemize}\setlength{\itemsep}{0.15em}
\item $\mathcal{V},\mathcal{E}$: current vertex and edge sets.
\item $\mathbf{x}\in\mathcal{V}$: a vertex position in $\mathbb{R}^3$.
\item $m$ indexes ReLU layers and $k$ indexes neurons within layer $m$.
\item $d_i=\nu^{(m)}_k(\mathbf{x}_i)$: pre-activation at $\mathbf{x}_i$.
\item $\mathcal{H}^{(m)}_k:=\{\mathbf{x}\mid \nu^{(m)}_k(\mathbf{x})=0\}$: decision boundary (folded hyperplane) in the input space induced by neuron $(m,k)$.
\item $\hat{\mathbf{x}}$: an intersection point created by subdividing an edge crossing $\mathcal{H}^{(m)}_k$.
\item $\gamma^{(m)}_k(\mathbf{x})\in\{-1,0,+1\}$: sign of $\nu^{(m)}_k(\mathbf{x})$ with threshold $\epsilon_s$ for numerical stability.
\item $\mathbf{s}^{(m)}_k(\mathbf{x})=[\gamma^{(1)}_1(\mathbf{x}),\ldots,\gamma^{(m)}_k(\mathbf{x})]$: sign vector of $\mathbf{x}$ up to neuron $(m,k)$.
\item $\textsc{ConnectIntersections}(\mathcal{S})$: connects intersection points on $\mathcal{H}^{(m)}_k$ into edges by tracking adjacent regions via sign-vector perturbation (Secs.~\ref{sec:edge-subdivision} and~\ref{sec:preliminary-perturbation}).
\end{itemize}
\vspace{0.4em}

\noindent\textbf{Intuition}~~
Edge subdivision incrementally constructs a 1-skeleton (vertices and edges) of the ReLU-induced polyhedral complex in the input space by sequentially processing neuron boundaries in the order of $(m,k)$.
Processing one neuron $(m,k)$ via its boundary $\mathcal{H}^{(m)}_k$, we (i) add intersection points where edges cross $\mathcal{H}^{(m)}_k$, and (ii) connect intersection points that lie on the same face of a linear region.

\begin{center}
\begin{minipage}{\linewidth}
\captionsetup{type=algorithm}
\captionof{algorithm}{Edge Subdivision for single neuron $(m,k)$}
\label{alg:edgesubdiv_single}

\vspace{-0.3em}
\hrule width\linewidth
\vspace{0.35em}

\small
\textbf{Input:} vertices $\mathcal{V}$, edges $\mathcal{E}$, neuron $(m,k)$ (pre-activation $\nu^{(m)}_k(\cdot)$), threshold $\epsilon_s$\\
\textbf{Output:} updated $(\mathcal{V},\mathcal{E})$
\vspace{0.35em}

\begin{algorithmic}[1]
\State $\mathcal{S}\gets\emptyset$ \Comment{$(\hat{\mathbf{x}},\mathbf{s}^{(m)}_k(\hat{\mathbf{x}}))$ on $\mathcal{H}^{(m)}_k$}
\State $\mathcal{E}_{\mathrm{old}}\gets \mathcal{E}$ \Comment{snapshot}
\ForAll{$(\mathbf{x}_0,\mathbf{x}_1)\in\mathcal{E}_{\mathrm{old}}$}
  \State $\gamma_0\gets \gamma^{(m)}_k(\mathbf{x}_0),\;\gamma_1\gets \gamma^{(m)}_k(\mathbf{x}_1)$
  \If{$\gamma_0\gamma_1<0$}
    \State $d_0\gets \nu^{(m)}_k(\mathbf{x}_0),\; d_1\gets \nu^{(m)}_k(\mathbf{x}_1)$
    \State $w \gets |d_0|/|d_0-d_1|$
    \State $\hat{\mathbf{x}} \gets (1-w)\mathbf{x}_0 + w\mathbf{x}_1$
    \State $\mathcal{V}\gets \mathcal{V}\cup\{\hat{\mathbf{x}}\}$
    \State $\mathcal{E}\gets (\mathcal{E}\setminus\{(\mathbf{x}_0,\mathbf{x}_1)\}) \cup \{(\mathbf{x}_0,\hat{\mathbf{x}}),(\hat{\mathbf{x}},\mathbf{x}_1)\}$
    \State $\mathcal{S}\gets \mathcal{S}\cup\{(\hat{\mathbf{x}},\mathbf{s}^{(m)}_k(\hat{\mathbf{x}}))\}$
  \EndIf
\EndFor
\State $\mathcal{E}\gets \mathcal{E}\cup \Call{ConnectIntersections}{\mathcal{S}}$
\State \Return $(\mathcal{V},\mathcal{E})$
\end{algorithmic}

\normalsize
\vspace{0.25em}
\hrule width\linewidth
\vspace{0.35em}
\end{minipage}
\end{center}

\vspace{0.3em}
\noindent\textbf{Remark}~~
In TetraSDF, region tracking must account for both ReLU MLP linear regions and encoder-induced polyhedral cells.
We therefore replace the pure sign-vector test with the augmented indicator $\mathbf{r}(\mathbf{x}) \oplus \mathbf{s}^{(m)}_k(\mathbf{x})$ to specify which region $\mathbf{x}$ lies in, and use barycentric masks $\mathbf{m}(\mathbf{x})$ to check whether $\mathbf{x}$ lies on a cell boundary and to determine which boundary it lies on (Secs.~\ref{sec:region-indicators} and~\ref{sec:perturbation}).
This yields our grid-aware edge subdivision and preserves the same high-level subdivision rule.
\section{Neighbor Configuration Lookup Tables for the Tetrahedral Grid}
\label{supp:neighbor_lut}
We report the lookup tables of neighbor configurations used in \cref{sec:perturbation} for the tetrahedral grid in \cref{fig:tetsubdiv_mtd}, divided into face, edge, and vertex neighbors.
Each tetrahedron $\mathcal{T}$ has ordered vertices $(v_0, v_1, v_2, v_3)$.
We fix the local face and edge indices as follows:
face index $f \in \{0,1,2,3\}$ refers to the face opposite vertex $v_f$ (\eg, $f = 0$ corresponds to the face $\{v_1, v_2, v_3\}$), and edge index $e \in \{0,\dots,5\}$ refers to
\begin{equation}
e_0 = \{v_0, v_1\},\;
e_1 = \{v_0, v_2\},\;
e_2 = \{v_0, v_3\},\;
e_3 = \{v_1, v_2\},\;
e_4 = \{v_1, v_3\},\;
e_5 = \{v_2, v_3\}.
\end{equation}
For each configuration, the neighbor index $i$ runs over $\{0,1,\dots\}$, where $i=0$ denotes the self configuration $(\Delta_0{=}(0,0,0),\, t_0{=}t)$.
In the face and edge neighbor tables, we list only the non-trivial neighbors with indices $i \ge 1$ and do not explicitly include the $i=0$ row.
For the vertex neighbor table, the lookup is defined per cube corner; for a given cube corner $c$ and tetra offset $t$, one of these rows may coincide with the self configuration.

\begin{table}[htbp]
\centering
\renewcommand{\arraystretch}{1.1}
\setlength{\tabcolsep}{5pt}
\caption{
Face neighbor lookup table. For each tetra offset $t \in \{0,\dots,5\}$ and local face index  $f \in \{0,\dots,3\}$, we list the neighbor index $i \ge 1$, the anchor displacement $\Delta_i = (\Delta_{x,i}, \Delta_{y,i}, \Delta_{z,i}) \in \mathbb{Z}^3$, and the neighbor tetra offset $t_i$.
In the table, the columns $\Delta_x,\Delta_y,\Delta_z$ give the components of $\Delta_i$ for each row $i$.
These entries correspond to the configurations $(\Delta_i, t_i)$ used in~\cref{sec:perturbation}.
The self configuration $(i{=}0,\Delta_0{=}(0,0,0),t_0{=}t)$ is omitted.
}
\begin{tabular}{cc|c|cccc|cc|c|cccc}
\toprule
\cmidrule(r){1-7}\cmidrule(l){8-14}
Input & $f$ & $i$ & $\Delta_x$ & $\Delta_y$ & $\Delta_z$ & $t_i$ &
Input & $f$ & $i$ & $\Delta_x$ & $\Delta_y$ & $\Delta_z$ & $t_i$\\
\midrule
$t{=}0$ & 0 & 1 & 1 & 0 & 0 & 4 & $t{=}3$ & 0 & 1 & 0 & 1 & 0 & 5\\
        & 1 & 1 & 0 & 0 & 0 & 5 &          & 1 & 1 & 0 & 0 & 0 & 4\\
        & 2 & 1 & 0 & 0 & 0 & 1 &          & 2 & 1 & 0 & 0 & 0 & 2\\
        & 3 & 1 & 0 & $-1$ & 0 & 2 &       & 3 & 1 & $-1$ & 0 & 0 & 1\\
\midrule
$t{=}1$ & 0 & 1 & 1 & 0 & 0 & 3 & $t{=}4$ & 0 & 1 & 0 & 0 & 1 & 2\\
        & 1 & 1 & 0 & 0 & 0 & 2 &          & 1 & 1 & 0 & 0 & 0 & 3\\
        & 2 & 1 & 0 & 0 & 0 & 0 &          & 2 & 1 & 0 & 0 & 0 & 5\\
        & 3 & 1 & 0 & 0 & $-1$ & 5 &       & 3 & 1 & $-1$ & 0 & 0 & 0\\
\midrule
$t{=}2$ & 0 & 1 & 0 & 1 & 0 & 0 & $t{=}5$ & 0 & 1 & 0 & 0 & 1 & 1\\
        & 1 & 1 & 0 & 0 & 0 & 1 &          & 1 & 1 & 0 & 0 & 0 & 0\\
        & 2 & 1 & 0 & 0 & 0 & 3 &          & 2 & 1 & 0 & 0 & 0 & 4\\
        & 3 & 1 & 0 & 0 & $-1$ & 4 &       & 3 & 1 & 0 & $-1$ & 0 & 3\\
\bottomrule
\end{tabular}
\end{table}

\begin{table*}[t]
\centering
\small
\renewcommand{\arraystretch}{1.05}
\setlength{\tabcolsep}{4pt}
\caption{
Edge neighbor lookup tables for tetra offsets $t = 0\text{--}5$.
For each input edge $(t,e)$ and neighbor index $i \ge 1$, we list the anchor displacement $\Delta_i = (\Delta_{x,i}, \Delta_{y,i}, \Delta_{z,i})$ and the neighbor tetra offset $t_i$.
In the table, the columns $\Delta_x,\Delta_y,\Delta_z$ give the components of $\Delta_i$ for each row $i$.
These entries correspond to the configurations $(\Delta_i, t_i)$ used in~\cref{sec:perturbation}.
The self configuration $(i{=}0,\Delta_0{=}(0,0,0),t_0{=}t)$ is omitted.
}
\resizebox{\textwidth}{!}{%
\begin{tabular}{ccc}
\begin{minipage}[t]{0.36\textwidth}
\centering
\begin{tabular}{c|c|cccc}
\toprule
Input & $i$ & $\Delta_x$ & $\Delta_y$ & $\Delta_z$ & $t_i$\\
\midrule
$t{=}0,\,e{=}0$ & 1 & 0 & $-1$ & $-1$ & 3 \\
                & 2 & 0 & $-1$ & $-1$ & 4 \\
                & 3 & 0 & 0 & $-1$ & 5 \\
                & 4 & 0 & $-1$ & 0 & 2 \\
                & 5 & 0 & 0 & 0 & 1 \\
\midrule
$t{=}0,\,e{=}1$ & 1 & 0 & $-1$ & 0 & 2 \\
                & 2 & 0 & $-1$ & 0 & 3 \\
                & 3 & 0 & 0 & 0 & 5 \\
\midrule
$t{=}0,\,e{=}2$ & 1 & 0 & 0 & 0 & 1 \\
                & 2 & 0 & 0 & 0 & 2 \\
                & 3 & 0 & 0 & 0 & 3 \\
                & 4 & 0 & 0 & 0 & 4 \\
                & 5 & 0 & 0 & 0 & 5 \\
\midrule
$t{=}0,\,e{=}3$ & 1 & 0 & $-1$ & 0 & 1 \\
                & 2 & 0 & $-1$ & 0 & 2 \\
                & 3 & 1 & $-1$ & 0 & 3 \\
                & 4 & 1 & 0 & 0 & 4 \\
                & 5 & 1 & 0 & 0 & 5 \\
\midrule
$t{=}0,\,e{=}4$ & 1 & 0 & 0 & 0 & 1 \\
                & 2 & 1 & 0 & 0 & 3 \\
                & 3 & 1 & 0 & 0 & 4 \\
\midrule
$t{=}0,\,e{=}5$ & 1 & 0 & 0 & 0 & 5 \\
                & 2 & 1 & 0 & 0 & 4 \\
                & 3 & 0 & 0 & 1 & 1 \\
                & 4 & 1 & 0 & 1 & 2 \\
                & 5 & 1 & 0 & 1 & 3 \\
\bottomrule
\end{tabular}
\end{minipage}
&
\begin{minipage}[t]{0.36\textwidth}
\centering
\begin{tabular}{c|c|cccc}
\toprule
Input & $i$ & $\Delta_x$ & $\Delta_y$ & $\Delta_z$ & $t_i$\\
\midrule
$t{=}1,\,e{=}0$ & 1 & 0 & $-1$ & $-1$ & 3 \\
                & 2 & 0 & $-1$ & $-1$ & 4 \\
                & 3 & 0 & 0 & $-1$ & 5 \\
                & 4 & 0 & $-1$ & 0 & 2 \\
                & 5 & 0 & 0 & 0 & 0 \\
\midrule
$t{=}1,\,e{=}1$ & 1 & 0 & 0 & $-1$ & 4 \\
                & 2 & 0 & 0 & $-1$ & 5 \\
                & 3 & 0 & 0 & 0 & 2 \\
\midrule
$t{=}1,\,e{=}2$ & 1 & 0 & 0 & 0 & 0 \\
                & 2 & 0 & 0 & 0 & 2 \\
                & 3 & 0 & 0 & 0 & 3 \\
                & 4 & 0 & 0 & 0 & 4 \\
                & 5 & 0 & 0 & 0 & 5 \\
\midrule
$t{=}1,\,e{=}3$ & 1 & 0 & 0 & $-1$ & 0 \\
                & 2 & 0 & 0 & $-1$ & 5 \\
                & 3 & 1 & 0 & $-1$ & 4 \\
                & 4 & 1 & 0 & 0 & 2 \\
                & 5 & 1 & 0 & 0 & 3 \\
\midrule
$t{=}1,\,e{=}4$ & 1 & 0 & 0 & 0 & 0 \\
                & 2 & 1 & 0 & 0 & 3 \\
                & 3 & 1 & 0 & 0 & 4 \\
\midrule
$t{=}1,\,e{=}5$ & 1 & 0 & 0 & 0 & 2 \\
                & 2 & 1 & 0 & 0 & 3 \\
                & 3 & 0 & 1 & 0 & 0 \\
                & 4 & 1 & 1 & 0 & 4 \\
                & 5 & 1 & 1 & 0 & 5 \\
\bottomrule
\end{tabular}
\end{minipage}
&
\begin{minipage}[t]{0.36\textwidth}
\centering
\begin{tabular}{c|c|cccc}
\toprule
Input & $i$ & $\Delta_x$ & $\Delta_y$ & $\Delta_z$ & $t_i$\\
\midrule
$t{=}2,\,e{=}0$ & 1 & $-1$ & 0 & $-1$ & 0 \\
                & 2 & $-1$ & 0 & $-1$ & 5 \\
                & 3 & 0 & 0 & $-1$ & 4 \\
                & 4 & $-1$ & 0 & 0 & 1 \\
                & 5 & 0 & 0 & 0 & 3 \\
\midrule
$t{=}2,\,e{=}1$ & 1 & 0 & 0 & $-1$ & 4 \\
                & 2 & 0 & 0 & $-1$ & 5 \\
                & 3 & 0 & 0 & 0 & 1 \\
\midrule
$t{=}2,\,e{=}2$ & 1 & 0 & 0 & 0 & 0 \\
                & 2 & 0 & 0 & 0 & 1 \\
                & 3 & 0 & 0 & 0 & 3 \\
                & 4 & 0 & 0 & 0 & 4 \\
                & 5 & 0 & 0 & 0 & 5 \\
\midrule
$t{=}2,\,e{=}3$ & 1 & 0 & 0 & $-1$ & 3 \\
                & 2 & 0 & 0 & $-1$ & 4 \\
                & 3 & 0 & 1 & $-1$ & 5 \\
                & 4 & 0 & 1 & 0 & 0 \\
                & 5 & 0 & 1 & 0 & 1 \\
\midrule
$t{=}2,\,e{=}4$ & 1 & 0 & 0 & 0 & 3 \\
                & 2 & 0 & 1 & 0 & 0 \\
                & 3 & 0 & 1 & 0 & 5 \\
\midrule
$t{=}2,\,e{=}5$ & 1 & 0 & 0 & 0 & 1 \\
                & 2 & 1 & 0 & 0 & 3 \\
                & 3 & 0 & 1 & 0 & 0 \\
                & 4 & 1 & 1 & 0 & 4 \\
                & 5 & 1 & 1 & 0 & 5 \\
\bottomrule
\end{tabular}
\end{minipage}
\end{tabular}
}
\end{table*}

\begin{table*}[t]
\centering
\small
\renewcommand{\arraystretch}{1.05}
\setlength{\tabcolsep}{4pt}
\resizebox{\textwidth}{!}{%
\begin{tabular}{ccc}
\begin{minipage}[t]{0.36\textwidth}
\centering
\begin{tabular}{c|c|cccc}
\toprule
Input & $i$ & $\Delta_x$ & $\Delta_y$ & $\Delta_z$ & $t_i$\\
\midrule
$t{=}3,\,e{=}0$ & 1 & $-1$ & 0 & $-1$ & 0 \\
                & 2 & $-1$ & 0 & $-1$ & 5 \\
                & 3 & 0 & 0 & $-1$ & 4 \\
                & 4 & $-1$ & 0 & 0 & 1 \\
                & 5 & 0 & 0 & 0 & 2 \\
\midrule
$t{=}3,\,e{=}1$ & 1 & $-1$ & 0 & 0 & 0 \\
                & 2 & $-1$ & 0 & 0 & 1 \\
                & 3 & 0 & 0 & 0 & 4 \\
\midrule
$t{=}3,\,e{=}2$ & 1 & 0 & 0 & 0 & 0 \\
                & 2 & 0 & 0 & 0 & 1 \\
                & 3 & 0 & 0 & 0 & 2 \\
                & 4 & 0 & 0 & 0 & 4 \\
                & 5 & 0 & 0 & 0 & 5 \\
\midrule
$t{=}3,\,e{=}3$ & 1 & $-1$ & 0 & 0 & 1 \\
                & 2 & $-1$ & 0 & 0 & 2 \\
                & 3 & $-1$ & 1 & 0 & 0 \\
                & 4 & 0 & 1 & 0 & 4 \\
                & 5 & 0 & 1 & 0 & 5 \\
\midrule
$t{=}3,\,e{=}4$ & 1 & 0 & 0 & 0 & 2 \\
                & 2 & 0 & 1 & 0 & 0 \\
                & 3 & 0 & 1 & 0 & 5 \\
\midrule
$t{=}3,\,e{=}5$ & 1 & 0 & 0 & 0 & 4 \\
                & 2 & 0 & 1 & 0 & 5 \\
                & 3 & 0 & 0 & 1 & 2 \\
                & 4 & 0 & 1 & 1 & 0 \\
                & 5 & 0 & 1 & 1 & 1 \\
\bottomrule
\end{tabular}
\end{minipage}
&
\begin{minipage}[t]{0.36\textwidth}
\centering
\begin{tabular}{c|c|cccc}
\toprule
Input & $i$ & $\Delta_x$ & $\Delta_y$ & $\Delta_z$ & $t_i$\\
\midrule
$t{=}4,\,e{=}0$ & 1 & $-1$ & $-1$ & 0 & 1 \\
                & 2 & $-1$ & $-1$ & 0 & 2 \\
                & 3 & 0 & $-1$ & 0 & 3 \\
                & 4 & $-1$ & 0 & 0 & 0 \\
                & 5 & 0 & 0 & 0 & 5 \\
\midrule
$t{=}4,\,e{=}1$ & 1 & $-1$ & 0 & 0 & 0 \\
                & 2 & $-1$ & 0 & 0 & 1 \\
                & 3 & 0 & 0 & 0 & 3 \\
\midrule
$t{=}4,\,e{=}2$ & 1 & 0 & 0 & 0 & 0 \\
                & 2 & 0 & 0 & 0 & 1 \\
                & 3 & 0 & 0 & 0 & 2 \\
                & 4 & 0 & 0 & 0 & 3 \\
                & 5 & 0 & 0 & 0 & 5 \\
\midrule
$t{=}4,\,e{=}3$ & 1 & $-1$ & 0 & 0 & 0 \\
                & 2 & $-1$ & 0 & 0 & 5 \\
                & 3 & $-1$ & 0 & 1 & 1 \\
                & 4 & 0 & 0 & 1 & 2 \\
                & 5 & 0 & 0 & 1 & 3 \\
\midrule
$t{=}4,\,e{=}4$ & 1 & 0 & 0 & 0 & 5 \\
                & 2 & 0 & 0 & 1 & 1 \\
                & 3 & 0 & 0 & 1 & 2 \\
\midrule
$t{=}4,\,e{=}5$ & 1 & 0 & 0 & 0 & 3 \\
                & 2 & 0 & 1 & 0 & 5 \\
                & 3 & 0 & 0 & 1 & 2 \\
                & 4 & 0 & 1 & 1 & 0 \\
                & 5 & 0 & 1 & 1 & 1 \\
\bottomrule
\end{tabular}
\end{minipage}
&
\begin{minipage}[t]{0.36\textwidth}
\centering
\begin{tabular}{c|c|cccc}
\toprule
Input & $i$ & $\Delta_x$ & $\Delta_y$ & $\Delta_z$ & $t_i$\\
\midrule
$t{=}5,\,e{=}0$ & 1 & $-1$ & $-1$ & 0 & 1 \\
                & 2 & $-1$ & $-1$ & 0 & 2 \\
                & 3 & 0 & $-1$ & 0 & 3 \\
                & 4 & $-1$ & 0 & 0 & 0 \\
                & 5 & 0 & 0 & 0 & 4 \\
\midrule
$t{=}5,\,e{=}1$ & 1 & 0 & $-1$ & 0 & 2 \\
                & 2 & 0 & $-1$ & 0 & 3 \\
                & 3 & 0 & 0 & 0 & 0 \\
\midrule
$t{=}5,\,e{=}2$ & 1 & 0 & 0 & 0 & 0 \\
                & 2 & 0 & 0 & 0 & 1 \\
                & 3 & 0 & 0 & 0 & 2 \\
                & 4 & 0 & 0 & 0 & 3 \\
                & 5 & 0 & 0 & 0 & 4 \\
\midrule
$t{=}5,\,e{=}3$ & 1 & 0 & $-1$ & 0 & 3 \\
                & 2 & 0 & $-1$ & 0 & 4 \\
                & 3 & 0 & $-1$ & 1 & 2 \\
                & 4 & 0 & 0 & 1 & 0 \\
                & 5 & 0 & 0 & 1 & 1 \\
\midrule
$t{=}5,\,e{=}4$ & 1 & 0 & 0 & 0 & 4 \\
                & 2 & 0 & 0 & 1 & 1 \\
                & 3 & 0 & 0 & 1 & 2 \\
\midrule
$t{=}5,\,e{=}5$ & 1 & 0 & 0 & 0 & 0 \\
                & 2 & 1 & 0 & 0 & 4 \\
                & 3 & 0 & 0 & 1 & 1 \\
                & 4 & 1 & 0 & 1 & 2 \\
                & 5 & 1 & 0 & 1 & 3 \\
\bottomrule
\end{tabular}
\end{minipage}
\end{tabular}
}
\end{table*}

\begin{table}[h]
\centering
\renewcommand{\arraystretch}{1.1}
\setlength{\tabcolsep}{5pt}
\caption{
Vertex neighbor lookup table around cube corner $(0,0,0)$.
For each neighbor index $i$, we list the anchor displacement
$\Delta_i = (\Delta_{x,i}, \Delta_{y,i}, \Delta_{z,i})$ and the neighbor tetra offset $t_i$.
These entries correspond to the configurations $(\Delta_i, t_i)$ used in~\cref{sec:perturbation}.
For a general cube corner $c = (c_x,c_y,c_z)\in\{0,1\}^3$, we obtain the corresponding offsets by replacing $\Delta_i \leftarrow \Delta_i + c$.
Row index ranges (\eg, $1\text{--}6$) indicate consecutive neighbor indices that share the same $(\Delta_i, t_i)$ configuration.
Note that, unlike the face and edge tables, this vertex table is defined per cube corner for convenience, since the vertices of each tetrahedron can be mapped to a cube corner.
For a specific cube corner $c$ and tetra offset $t$, one of these rows may coincide with the self configuration $(\Delta_0{=}(0,0,0), t_0{=}t)$.
}
\begin{tabular}{c|ccc|c}
\toprule
$i$ & $\Delta_x$ & $\Delta_y$ & $\Delta_z$ & $t_i$ \\
\midrule
1--6   & $-1$ & $-1$ & $-1$ & $0\text{--}5$ \\
7--8   & $-1$ & $-1$ &  0   & $1,2$ \\
9--10  & $-1$ &  0   & $-1$ & $0,5$ \\
11--12 & $-1$ &  0   &  0   & $0,1$ \\
13--14 &  0   & $-1$ & $-1$ & $3,4$ \\
15--16 &  0   & $-1$ &  0   & $2,3$ \\
17--18 &  0   &  0   & $-1$ & $4,5$ \\
19--24 &  0   &  0   &  0   & $0\text{--}5$ \\
\bottomrule
\end{tabular}
\end{table}

}
\end{document}